\documentclass{article}

\usepackage[main, final]{neurips_2026}
\let\cite\citep

\usepackage[utf8]{inputenc}
\usepackage[T1]{fontenc}
\usepackage{hyperref}
\usepackage{url}
\usepackage{booktabs}
\usepackage{amsfonts}
\usepackage{nicefrac}
\usepackage{microtype}
\usepackage{xcolor}
\usepackage{wrapfig}
\usepackage{adjustbox}

\usepackage{algorithm}
\usepackage{algorithmic}

\usepackage{subcaption}
\usepackage{booktabs}
\usepackage{graphicx}
\usepackage{amsmath}
\usepackage{amsthm}
\usepackage{amsfonts}
\usepackage[dvipsnames]{xcolor}
\usepackage{booktabs}
\usepackage[toc,page,header]{appendix}
\usepackage{tikz}
\usepackage{bbm}
\usepackage{rotating}

\usetikzlibrary{shapes.geometric}
\DeclareRobustCommand{\yellowstar}{\tikz[baseline=-0.7ex]\node[star,star points=5,star point ratio=2.5,fill=yellow,draw=black,inner sep=1.2pt,scale=0.8] {};}

  \definecolor{alggreen}{RGB}{34, 139, 34}
  \newcommand{\cmt}[1]{{\color{alggreen}\texttt{#1}}}

\usepackage{amsmath}
\usepackage{amssymb}
\usepackage{mathtools}
\usepackage{amsthm}

\def\UC{UC}

\usepackage[capitalize,noabbrev]{cleveref}

\theoremstyle{plain}
\newtheorem{theorem}{Theorem}[section]
\newtheorem{proposition}[theorem]{Proposition}
\newtheorem{lemma}[theorem]{Lemma}
\newtheorem{corollary}[theorem]{Corollary}
\theoremstyle{definition}

\newtheorem{assumption}[theorem]{Assumption}
\theoremstyle{remark}
\newtheorem{remark}[theorem]{Remark}

\title{Active Learning for Gaussian Process Regression Under Self-Induced Boltzmann Weights}

\author{
  Jixiang Qing\thanks{MARS: Mathematics for AI in Real-world Systems, School of Mathematical Sciences, Lancaster University, Lancaster, LA1 4YF, United Kingdom. Correspondence to: \texttt{\{j.qing,\ m.sachs\}@lancaster.ac.uk}}
  \quad\quad Henry Moss\footnotemark[1]
  \quad\quad Matthias Sachs\footnotemark[1]
}

\begin{document}

\maketitle

\begin{abstract}
  We consider the active learning problem where the goal is to learn an unknown function with low prediction error under an unknown Boltzmann distribution induced by the function itself. This self-induced weighting arises naturally in problems such as potential energy surface (PES) modeling in computational chemistry,
  yet poses unique challenges as the target distribution is unknown and its partition function is intractable. We propose \texttt{AB-SID-iVAR}, a Gaussian Process-based acquisition function that approximates the intractable Bayesian target distribution in closed form while avoiding partition function estimation, and is applicable to both discrete and continuous input domains. We also analyze a Thompson sampling alternative (\texttt{TS-SID-iVAR}) as a higher variance Monte Carlo variant. Despite the unknown target, under mild conditions, we establish that the terminal prediction error vanishes with high probability, and provide a tighter average-case guarantee. We demonstrate consistent improvements over existing approaches in this setting on synthetic benchmarks and real-world PES modeling and drug discovery tasks. Our code is available at \url{https://github.com/TsingQAQ/SIDAL}.
\end{abstract}

\setcounter{section}{1}
\section*{1. Introduction}

Learning an \textit{unknown function} from data is a fundamental problem in statistics and machine learning. When function evaluations are expensive, sequential experimental design \cite{pukelsheim2006optimal, fedorov2013theory, huan2024optimal} (also known as Active Learning (AL) \cite{settles2009active, rainforth2024modern}) provides a principled approach to optimize data acquisition based on a predictive surrogate.
Gaussian processes \cite{Rasmussen06} are a popular choice (\cite{huang2026loss, smith2023prediction, qing2022adaptive}) with convergence guarantees under mild regularity conditions \cite{krause2008near, srinivas2009gaussian, vakili2021scalable}.

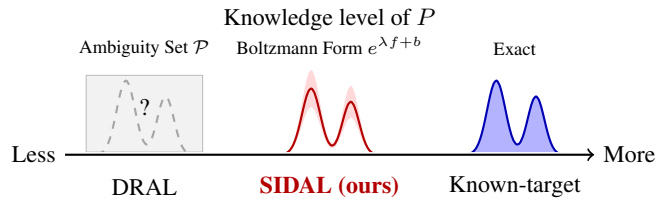
\begin{wrapfigure}[14]{r}{0.6\textwidth}
\vspace{-1.2em}
\centering
\begin{tikzpicture}[scale=0.7]

    \node[font=\footnotesize] at (5,2.55) {Knowledge level of $P$};

    \node[align=center, font=\scriptsize] at (1.5,2.) {Ambiguity Set $\mathcal{P}$};
    \node[align=center, font=\scriptsize] at (5,2.07) {Boltzmann Form $e^{\lambda f + b}$};
    \node[align=center, font=\scriptsize] at (8.5,2.) {Exact};

    \begin{scope}[shift={(1.5,0)}]
        \draw[gray!50, fill=gray!10] (-1.1,0.05) rectangle (1.1,1.5);
        \draw[dashed, thick, gray!70] plot[smooth, domain=-0.9:0.9, samples=30] (\x, {1.4*exp(-((\x+0.35)^2)/0.06) + 1.1*exp(-((\x-0.4)^2)/0.05)});
        \node[font=\small] at (0,0.9) {?};
    \end{scope}

    \begin{scope}[shift={(5,0)}]
        \fill[red!15] plot[smooth, domain=-0.9:0.9, samples=30] (\x, {1.6*exp(-((\x+0.35)^2)/0.06) + 1.3*exp(-((\x-0.4)^2)/0.05)}) -- plot[smooth, domain=0.9:-0.9, samples=30] (\x, {0.9*exp(-((\x+0.35)^2)/0.06) + 0.7*exp(-((\x-0.4)^2)/0.05)}) -- cycle;
        \draw[thick, red!70!black] plot[smooth, domain=-0.9:0.9, samples=30] (\x, {1.25*exp(-((\x+0.35)^2)/0.06) + 1.0*exp(-((\x-0.4)^2)/0.05)});
    \end{scope}

    \begin{scope}[shift={(8.5,0)}]
        \fill[blue!30] plot[smooth, domain=-0.9:0.9, samples=30] (\x, {1.4*exp(-((\x+0.35)^2)/0.06) + 1.1*exp(-((\x-0.4)^2)/0.05)}) -- (0.9,0) -- (-0.9,0) -- cycle;
        \draw[thick, blue!70!black] plot[smooth, domain=-0.9:0.9, samples=30] (\x, {1.4*exp(-((\x+0.35)^2)/0.06) + 1.1*exp(-((\x-0.4)^2)/0.05)});
    \end{scope}

    \draw[-, thick, white, line width=2pt] (0,0) -- (10,0);
    \draw[->, thick] (0,0) -- (10,0);
    \node[left, font=\footnotesize] at (0,0) {Less};
    \node[right, font=\footnotesize] at (10,0) {More};

    \node[font=\footnotesize] at (1.5,-0.6) {DRAL};
    \node[font=\footnotesize\bfseries, red!70!black] at (5,-0.6) {SIDAL (ours)};
    \node[font=\footnotesize] at (8.5,-0.6) {Known-target};
\end{tikzpicture}
\caption{Spectrum of AL methods by knowledge of the target distribution $P_f$, where the goal is to minimize prediction error $\mathbb{E}_{x \sim P_f}[(f(x) - \hat{f}(x))^2]$.
\textbf{Left}: DRAL assumes $P_f$ lies in an ambiguity set.
\textbf{Middle}: SIDAL (ours) knows $P_f$ takes Boltzmann form $\exp(\lambda f + b)$ but $f$ itself is unknown.
\textbf{Right}: Known-target has exact $P_f$.}
\label{fig:knowledge_spectrum}
\end{wrapfigure}

Given a surrogate, an \textit{acquisition function} characterizes the utility
criterion motivated by the learning objective, such as minimizing prediction
error over input space weighted by a \textit{target distribution} $P$, with uniform distribution as the most common scenario. Under this scheme,
available knowledge about $P$ naturally shapes the acquisition strategy:
when $P$ is fully specified, weighted uncertainty reduction applies
\cite{seo2000gaussian, kirsch2021test, smith2023prediction}; when $P$ is unknown, one may instead specify possible $P$ through an ambiguity set $\mathcal{P}$ under a Distributionally Robust Active Learning (DRAL) framework \citep{Frogner, takeno2025distributionally}.

\begin{figure*}[t]
    \centering
\includegraphics[width=1.0\linewidth]{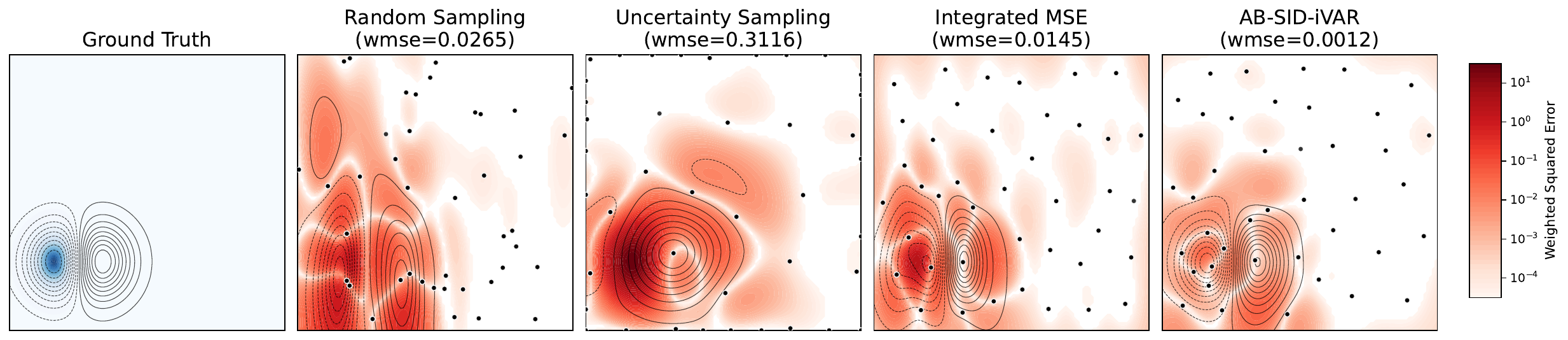}
    \caption{Comparison of active learning methods on the (continuous input) Gramacy 2D function \citep{gramacy2008gaussian}, aiming to minimize prediction error under a self-induced Boltzmann distribution $P_f \propto \exp(-f)$.
Left: ground truth with high-probability regions shaded in \textcolor{cyan}{blue}; accurate prediction is mainly required in these regions.
Right: weighted squared error (heatmap, log scale) after 40 queries (black dots: \raisebox{0.25ex}{\scalebox{0.6}{$\bullet$}}); contours show the learned GP posterior mean.
SID-unaware baselines include Random Sampling (RS), Uncertainty Sampling (US), and Integrated MSE (IMSE).
Our method \texttt{AB-SID-iVAR} balances exploration and concentration by placing samples in high-probability regions, achieving substantially lower weighted MSE (wmse) than baselines.}
    \label{fig:spotlight_fig}
\end{figure*}
Viewing the aforementioned variants on the unified spectrum (Figure~\ref{fig:knowledge_spectrum}) of knowledge about $P$, we argue that many applications naturally fall between established methods. For instance, in molecular modeling and free energy estimation, the distribution of interest takes the Boltzmann form \citep{frenkel2023understanding} weighted by the unknown energy function itself. We categorize this setting as \textbf{Active Learning under a Self-Induced Distribution (SIDAL)}, where the target distribution $P$ depends on the unknown function $f$. Specifically, $P$ often takes a Boltzmann form (Assumption~\ref{ass:target_dist}). To the best of our knowledge, this setting has not been formally studied in the active learning literature. While related heuristics have emerged in scientific domains, they require careful tuning and lack theoretical guarantees; we defer a detailed discussion to Appendix~\ref{App: Related_Work}.

\textbf{Contributions} Aiming to fill the aforementioned gap, we propose an SID-aware AL algorithm  under the GP surrogate model, which samples to reduce the weighted mean-squared-error (MSE) sequentially. Building upon the theoretical analysis of \cite{srinivas2009gaussian, cai2024kernelized, takeno2025distributionally},we provide sublinear convergence guarantees for the terminal prediction error in both high-probability and average (a.k.a., in-expectation) forms. Our contributions can be summarized as follows:

\begin{itemize}
      \item We formalize the novel problem of active learning under self-induced distributions (Section~\hyperlink{sec:problem_setting}{3.1}), capturing scenarios where the target distribution depends on the unknown function itself.
    \item We propose a simple-to-implement acquisition function \texttt{\texttt{AB-SID-iVAR}} (fully defined by Eq.~\ref{Eq: SID-iVAR} and Table~\ref{tab:methods}) that is free of partition function estimation and applies to both discrete and continuous input domains. A Thompson sampling variant (\texttt{TS-SID-iVAR}) offering a bias-variance tradeoff is also analyzed.
    \item We provide sublinear convergence guarantees under high-probability and average settings (Section~\hyperlink{sec:theoretical_analysis}{4.2}), explicitly accounting for Monte Carlo approximation error in continuous domains, and show (Section~\hyperlink{Sec:unification}{4.3}) that existing heuristics admit convergence analysis under our framework.

    \item We demonstrate competitive empirical performance on both synthetic (Figure~\ref{fig:spotlight_fig}, \ref{fig:synthetic_exp}) and real-world problems in PES modeling and drug discovery (Figure~\ref{fig:MD_Modeling}, \ref{fig:molecular_drug_discovery}).
\end{itemize}

\setcounter{section}{2}
\section*{2. Preliminaries}
\setcounter{theorem}{0}
\subsection*{2.1. Gaussian Processes}\label{Sec: GP}

We model the unknown function using Gaussian Process (GP) regression, a natural choice for active learning due to its closed-form posterior updates and calibrated uncertainty estimates. Given a positive semi-definite kernel $k: \mathcal{X} \times \mathcal{X} \to \mathbb{R}$, which is normalized as $k(x, x) \leq 1$, and data $\mathcal{D}_t = \{(x_i, y_i)\}_{i=1}^t$, the predictive distribution is:
\begin{equation}
\begin{aligned}
    \mu_t(x) = k(x, \mathbf{X}_t)(K_t + \tau^2 I)^{-1} \mathbf{y}_t, \
    k_t(x, x') = k(x, x') - k(x, \mathbf{X}_t)(K_t + \tau^2 I)^{-1} k(\mathbf{X}_t, x'),
\end{aligned}
\label{Eq: GP_posterior}
\end{equation}

where $\mathbf{X}_t = (x_1, \ldots, x_t)$ denotes the sequence of queried inputs, $\mathbf{y}_t = (y_1, \ldots, y_t)^\top$ the corresponding noisy observations (subjecting to zero-mean Gaussian observation noise with variance $\tau^2$), $K_t = k(\mathbf{X}_t, \mathbf{X}_t) \in \mathbb{R}^{t \times t}$ the kernel (Gram) matrix, and $\sigma_t^2(x) := k_t(x, x)$ the predictive variance at $x$ given the data $\mathcal{D}_t$.

\subsection*{2.2. Regularity Assumptions}
We assume the regularity of $f$ is consistent with the kernel $k$. We primarily consider the Bayesian setting where $f \sim \mathcal{GP}(0, k)$, and follow this assumption throughout unless stated otherwise.

\begin{assumption}[Kernel Regularity]\label{ass:kernel_regularity}
The kernel $k$ is stationary and four times differentiable.
\end{assumption}
\noindent This is satisfied by common kernels such as squared exponential
and Matérn-$\nu$ with $\nu > 2$. Under this assumption, sample paths of $f \sim \mathcal{GP}(0,k)$
are Lipschitz continuous with high probability
\citep{ghosal2006posterior, srinivas2009gaussian}, and the posterior mean $\mu_t$ is similarly Lipschitz \citep{vakili2021scalable, takeno2025distributionally}.
Together, these ensure the \textit{prediction error} $f(x) - \mu_t(x)$ is Lipschitz:

\begin{lemma}[Prediction Error Lipschitz Continuity {\cite{takeno2025distributionally}}]\label{lem:lip-pred}
Under Assumption~\ref{ass:kernel_regularity}
with compact input domain $\mathcal{X} \subset [0,r]^d$, $x \mapsto f(x) - \mu_t(x)$ is
$L$-Lipschitz with probability at least $1 - \delta$, where
$L = b\sqrt{\log(4ad/\delta)}$ for kernel-dependent constants $a, b$.
\end{lemma}

This enables a high-probability uniform prediction error guarantee using GP as a model of $f$ for continuous input:

\begin{lemma}[GP Confidence Bound]\label{lem:gp_confidence}
Pick $\delta \in (0,1)$ and $t \in \mathbb{N}$. Then with probability at least $1 - \delta$, for all $x \in \mathcal{X}$:
\begin{equation*}
|f(x) - \mu_{t-1}(x)| \leq \beta_t^{1/2} \sigma_{t-1}(x) + \epsilon_t,
\end{equation*}
where for finite $|\mathcal{X}| < \infty$: $\beta_t = 2\log(|\mathcal{X}|/\delta)$, $\epsilon_t = 0$; for continuous $\mathcal{X} \subset [0,r]^d$ under Assumption~\ref{ass:kernel_regularity}: $\beta_t = 2d\log(tdr + 1) + 2\log(2/\delta)$, $\epsilon_t = L/t$.
\end{lemma}

The non-negative $\epsilon_t$ in the continuous case is induced by approximating $f - \mu_{t-1}$ on $[0, r]^d$ using a Cartesian grid of spacing $1/(dt)$, and $\beta_t$ grows logarithmically with $t$ due to a union bound over $(tdr+1)^d$ grid points. The proof follows from \citet{srinivas2009gaussian} Lemma 5.1 for discrete $\mathcal{X}$, and \citet{takeno2025distributionally} when using the Lipschitz constant $L$ from Lemma~\ref{lem:lip-pred}.

\setcounter{section}{3}
\section*{3. Problem Setting and Performance Analysis}
\setcounter{theorem}{0}
\hypertarget{sec:problem_setting}{}
We query $f$ to obtain noisy observations $y = f(x) + \epsilon$ with $\epsilon \sim \mathcal{N}(0, \tau^2)$. Given a \textit{surrogate model} $\hat{f}(x)$ built from the current dataset $\mathcal{D}_t$, our goal is to sequentially select $T$ query points that minimize the expected squared error under a distribution induced by $f$ itself, which we term the \textit{Self-Induced Distribution} (SID) $P_f$:
\begin{equation}
\mathcal{L}(\hat{f}) = \mathbb{E}_{x \sim P_f} \left[\left(\hat{f}(x) - f(x)\right)^2 \right],
\label{Eq: problem}
\end{equation}
\noindent where we have the following assumption regarding the SID:

\begin{assumption}[Boltzmann Self-Induced Distribution]
\label{ass:target_dist}
The target distribution $P_f(x)$ takes the Boltzmann distribution form\footnote{We parameterize using $\lambda$ directly; this corresponds to $\lambda = -1/\tilde{\mathcal{T}}$ in the physical Boltzmann distribution with temperature parameter $\tilde{\mathcal{T}}>0$.} $P_f(x) = Z_f^{-1} \exp(\lambda f(x) + b(x))$ where: (i) $\lambda \neq 0$; (ii) the \textit{bias} function $b: \mathcal{X} \to \mathbb{R}$ is known and bounded; (iii) the \textit{partition function} $Z_f < \infty$.
\end{assumption}

We use the GP predictive mean $\hat{f}(x) = \mu_t(x)$ as our surrogate. While $\mu_t$ is not the Bayes-optimal estimator for Eq.~\ref{Eq: problem} (which otherwise requires intractable integrations), this gap vanishes as the posterior concentrates (Appendix~\ref{App: Non-Bayes Optimal}).

Different from standard AL or DRAL as elaborated in Figure~\ref{fig:knowledge_spectrum}, one key challenge here is that the SID $P_f$ itself depends on the unknown function $f$, creating a circular dependence dilemma. To establish that learning under SID is feasible, we derive upper bounds on $\mathcal{L}(\hat{f})$ by relating the unknown $P_f$ to the computable surrogate $P_{\mu_t}$. These upper bounds, presented in both high-probability and average forms, inform our algorithm design in the next section.

To first establish a high-probability bound, we decompose the density ratio between $P_f$ and $P_{\mu_t}$ into a pointwise Boltzmann factor difference and a partition function difference:

\begin{lemma}[Log-Density Ratio]
\label{lem:log_density_ratio}
Let $g_1, g_2: \mathcal{X} \to \mathbb{R}$ and define $P_{g_i} \propto e^{\lambda g_i + b}$ for $i = 1, 2$.
Define $\Delta(x) := |g_1(x) - g_2(x)|$ and $\Delta_{\max} := \sup_{x \in \mathcal{X}} \Delta(x)$. Then:
\begin{equation*}
\left|\log P_{g_1}(x) - \log P_{g_2}(x)\right| \leq |\lambda|(\Delta(x) + \Delta_{\max}).
\end{equation*}
\end{lemma}
We state this lemma abstractly for reuse later (in Lemma~\ref{lem:var_bound_ts}). Substituting $(g_1, g_2) = (f, \mu_t)$ and combining with the confidence bound (Lemma~\ref{lem:gp_confidence}), we convert expectations under the unknown $P_f$ to the computable $P_{\mu_T}$:

\begin{lemma}[High-Probability MSE to Variance Bound]
\label{lem:hp_mse_bound}
Pick $\delta \in (0,1)$. Under Assumption~\ref{ass:target_dist}, let $\beta_T, \epsilon_T$ be as in Lemma~\ref{lem:gp_confidence}.
Then with probability at least $1-\delta$:
\begin{equation*}
\begin{aligned}
    & \mathcal{L}(\mu_T) \leq e^{2|\lambda|(\beta_T^{1/2}\sigma_{\max,T} + \epsilon_T)}   \cdot \left(\beta_T \cdot \mathbb{E}_{x \sim P_{\mu_T}}\left[\sigma_T^2(x)\right] + 2\beta_T^{1/2}\epsilon_T + \epsilon_T^2\right),
\end{aligned}
\end{equation*}
where $\sigma_{\max,T} = \max_{x \in \mathcal{X}} \sigma_T(x)$.
\end{lemma}

The exponential prefactor arises from bounding the partition function ratio, which requires uniform control of $f - \mu_T$ over all of $\mathcal{X}$ and is challenging to avoid in high-probability analysis \citep{cai2024kernelized}. As we shall see in Section~\hyperlink{sec:theoretical_analysis}{4.2}, while convergence is unaffected, this introduces additional logarithmic factors in the terminal prediction error bound. To circumvent estimating the partition function ratio entirely, we take expectation over $f|\mathcal{D}_T$, which allows $Z_f/Z_{\mu_T}$ to be absorbed via Jensen's inequality. The Gaussian posterior structure then admits closed-form control via Moment Generating Function (MGF) identities, yielding the following average upper bound:

\begin{lemma}[Posterior-Averaged MSE  to Variance Bound]\label{lem:bayesian_mse}
Under Assumption~\ref{ass:target_dist}, the terminal MSE averaged via posterior $f \vert \mathcal{D}_T$ satisfies:
\begin{equation*}
\begin{aligned}
& \mathbb{E}_{f|\mathcal{D}_{T}}\left[\mathbb{E}_{P_f}\left[(f - \mu_T)^2\right]\right] \leq e^{2\lambda^2}(1+4\lambda^2) \cdot \mathbb{E}_{P_{\mu_T}}[\sigma_T^2(x)].
\end{aligned}
\end{equation*}
\end{lemma}
Without the $\beta_T$ factor, this yields a tighter terminal error bound that better reflects finite-query and average-case performance. Notably, both upper bounds reflect the intrinsic difficulty of the problem: a more concentrated target distribution makes accurate prediction in high-probability regions more challenging.

\setcounter{section}{4}

\section*{4. Methodology}
\setcounter{theorem}{0}
\subsection*{4.1. Proposed Algorithms}

Both upper bounds involve weighted posterior variance, suggesting variance reduction as the algorithmic objective. To this end, our proposed integrated Variance Reduction (iVAR) minimizes expected posterior variance over a tractable unnormalized SID surrogate $\tilde{p}^{\mathrm{u}}_{t-1}$ (we use $\tilde{p}_{t-1}$ for its normalized version):
\begin{equation}
   x_t = \arg\min_{x \in \bar{\mathcal{X}}_t} \int_{x^\star \in \mathcal{X}} \tilde{p}^{\mathrm{u}}_{t-1}(x^\star) \sigma_t^2(x^\star \mid x_t = x) \, dx^\star,
\label{Eq: SID-iVAR}
\end{equation}

where a \emph{potential set} $\bar{\mathcal{X}}_t$ restricts the search to uncertain regions to ensure sufficient exploration. We propose \texttt{\texttt{AB-SID-iVAR}}, with the surrogate $\tilde{p}^u_{t-1}$ defined in Table~\ref{tab:methods}.

\begin{figure}[h!]
\begin{minipage}[t]{0.50\textwidth}
\vspace{0pt}
\captionof{table}{SID-iVAR variants. \textbf{AB-SID}: closed-form zero-order Taylor approximation of $\mathbb{E}_{f|\mathcal{D}_{t-1}}[P_f]$ (Appendix~\ref{App: AB-SID}); biased, low-variance. \textbf{TS-SID}: single posterior sample as MC estimate; unbiased, high-variance. Both share constraint set $\bar{\mathcal{X}}_t$.}
\vspace{0.45em}
\label{tab:methods}
\centering
\small
\setlength{\tabcolsep}{3pt}
\renewcommand{\arraystretch}{1.7}
\begin{tabular}{@{}lcc@{}}
\toprule
 & \textbf{AB-SID} & \textbf{TS-SID} \\
\midrule
$\tilde{p}^{\mathrm{u}}_{t-1}$
 & $e^{\lambda\mu_{t-1} + \frac{\lambda^2}{2}\sigma_{t-1}^2 + b}$
 & $e^{\lambda\tilde{f}_{t-1} + b}$, $\tilde{f}_{t-1} \overset{\text{draw}}{\sim} f|\mathcal{D}_{t-1}$ \\
\midrule
\multicolumn{3}{c}{$\bar{\mathcal{X}}_t = \{x : \sigma_{t-1}^2(x) \geq \mathbb{E}_{\tilde{p}_{t-1}}[\sigma_{t-1}^2]\}$} \\
\bottomrule
\end{tabular}
\end{minipage}%
\hfill
\begin{minipage}[t]{0.47\textwidth}
\vspace{-1em}
\begin{algorithm}[H]
\caption{\textcolor{purple}{AB}/\textcolor{cyan}{TS}-SID-iVAR}
\label{alg:sid-ivar}
\begin{algorithmic}[1]
\REQUIRE $\mathcal{X}$, GP prior $(0, k)$, budget $T$, $\mathcal{D}_0$
\FOR{$t = 1, \ldots, T$}
    \STATE Compute $\mu_{t-1}$, $\sigma_{t-1}^2$ from $\mathcal{D}_{t-1}$
    \STATE Construct $\tilde{p}^{\mathrm{u}}_{t-1}$ via \textcolor{purple}{AB} or \textcolor{cyan}{TS}-SID
    \STATE Construct $\bar{\mathcal{X}}_t$ (Table~\ref{tab:methods})
    \STATE $x_t \leftarrow \arg\min_{x \in \bar{\mathcal{X}}_t} \int \tilde{p}^{\mathrm{u}}_{t-1}(x^\star) \sigma_t^2(x^\star \mid x) \, dx^\star$
    \STATE Observe $y_t = f(x_t) + \epsilon_t$; update $\mathcal{D}_t$
\ENDFOR
\STATE \textbf{return} $\mu_T$
\end{algorithmic}
\end{algorithm}
\end{minipage}
\end{figure}

\texttt{AB-SID} (Approximate Bayesian SID) approximates the intractable Bayesian SID $\mathbb{E}_{f|\mathcal{D}_{t-1}}[P_f]$ in closed form via a zero-order Taylor expansion (Appendix~\ref{App: acq_func_derivation}).\footnote{A first-order correction is available (Appendix~\ref{App: AB-SID}) with marginally better accuracy (Figure~\ref{fig:acq_approx_comparison}), but requires $\mathcal{O}(|\mathcal{X}|^2)$ additional covariance computations for discrete $\mathcal{X}$.} Alternatively, one can approximate this integral via a single Monte Carlo sample, yielding \texttt{TS-SID} (Thompson Sampling SID), which offers a bias-variance tradeoff (Table~\ref{tab:methods}). Both methods avoid estimation of the partition function.

While both algorithms apply directly in discrete domains, continuous $\mathcal{X} \subset [0,r]^d$ requires approximating both the constraint set $\bar{\mathcal{X}}_t$ and the integral in Eq.~\ref{Eq: SID-iVAR}. We use MCMC to sample from $\tilde{p}^{\mathrm{u}}_{t-1}$ and approximate the integral; for \texttt{TS-SID}, we additionally approximate sample trajectories $\tilde{f}_{t-1}$ via pathwise conditioning \citep{wilson2021pathwise}.\footnote{The constraint set $\bar{\mathcal{X}}_t$ is always non-empty under the approximation.} Importantly, our theoretical analysis is agnostic to the specific MCMC sampler (provided that the variance-threshold estimator satisfies the conditional unbiasedness and concentration conditions specified in Appendix~\ref{App: them_main_hp_finit_hp}).

\paragraph{Algorithm Summary.} The overall procedure is summarized in Algorithm~\ref{alg:sid-ivar}. The key design choices are: (1) the SID surrogate's form $\tilde{p}_{t-1}$ that approximates the unknown target distribution, and (2) the potential set $\bar{\mathcal{X}}_t$ that ensures exploration of
uncertain regions.
For continuous $\mathcal{X}$, the integral in Line 5 and the threshold in $\bar{\mathcal{X}}_t$ are approximated via MCMC samples from $\tilde{p}_{t-1}$.
\subsection*{4.2. Theoretical Analysis} \hypertarget{sec:theoretical_analysis}{}

We now show that iteratively optimizing over the surrogate $\tilde{p}_{t-1}$ yields vanishing prediction error under the true SID $P_f$, and establish finite-sample upper bounds on the terminal prediction error $\mathcal{L}(\mu_T)$, as well as the rate at which it converges to zero as $T \to \infty$. Our analysis depends on the maximum information gain $\gamma_T = \max_{x_1,\ldots,x_T} I(y_{1:T}; f)$, which is sublinear in $T$ for standard kernels \citep{srinivas2009gaussian}. We state results for continuous $\mathcal{X}$ directly, as discrete $\mathcal{X}$ is a simplified case with $\epsilon_T = \epsilon_{\mathrm{MC}} = 0$.
\paragraph{High-Probability Guarantee}

\begin{theorem}[High-Probability Terminal MSE Bound]\label{thm:hp_terminal}
Under Assumptions~\ref{ass:kernel_regularity} and~\ref{ass:target_dist}, let $\mathcal{X} \subset [0,r]^d$ be compact. Consider \texttt{AB-SID} or \texttt{TS-SID} selecting $x_t \in \bar{\mathcal{X}}_t$ (Table~\ref{tab:methods}), where the constraint threshold at each iteration is estimated via $M$ Monte Carlo samples. Then for any $\delta \in (0,1)$, with probability at least $1 - \delta$:
\begin{equation}
\begin{aligned}
 \mathcal{L}(\mu_T) \leq e^{2(\kappa+2)|\lambda|(\beta^{1/2}_T + \epsilon_T)} \cdot \left(\frac{C_1\beta_T\gamma_T}{T} + \beta_T\epsilon_{\mathrm{MC}} + 2\beta^{1/2}_T\epsilon_T + \epsilon_T^2\right),
\end{aligned}
\label{Eq: high_prob_upper_bound}
\end{equation}
where $\kappa = 1$ for \texttt{AB-SID}, $\kappa = 2$ for \texttt{TS-SID}, $C_1 = 2/\log(1+\tau^{-2})$, $\beta_T, \epsilon_T$ are the coefficients of the uniform confidence event defined in Appendix~\ref{App: them_main_hp_finit_hp} with total GP failure budget $\delta/2$, and the cumulative MC error $\epsilon_{\mathrm{MC}} = \frac{2\log(4/\delta)}{3T} + \sqrt{\frac{\log(4/\delta)}{2MT}}$. For finite $\mathcal{X}$ with exact threshold computation, $\epsilon_T = \epsilon_{\mathrm{MC}} = 0$.
\end{theorem}

The upper bound in Eq.~\ref{Eq: high_prob_upper_bound} consists of three components: (i) the information gain term $\frac{\gamma_T}{T}$; (ii) a cumulative MC error $\epsilon_{\mathrm{MC}}$ from approximating the constraint threshold via MC sampling; and (iii) terms involving $\epsilon_T$ from extending pointwise GP confidence bounds to the entire continuous domain. The probability $\delta$ is combined with the GP confidence event and the MC concentration event via a union bound. Unlike prior work \citep{takeno2025distributionally} which assumes exact computation, our analysis explicitly accounts for this MC error. The exponential prefactor, as elaborated earlier, arises from bounding the density ratio between $P_f$ and $P_{\mu_T}$, with \texttt{TS-SID} incurring a larger factor ($\kappa=2$) due to the additional Thompson sampling approximation.

Further analyzing each term in the upper bound yields the following convergence rate.

\begin{corollary}[High-Probability Terminal MSE Convergence Rate]
\label{cor:rate}
Under the conditions of Theorem~\ref{thm:hp_terminal},
suppose that $\gamma_T=\mathcal{O}((\log T)^q)$
for some fixed $q\geq 0$.
Then: (i) $M$ is constant $\Rightarrow$
$\mathcal{L}(\mu_T)=\mathcal{O}(T^{-1/2+\epsilon})$;
(ii) $M=\Omega(T)$ $\Rightarrow$
$\mathcal{L}(\mu_T)=\mathcal{O}(T^{-1+\epsilon})$,
for arbitrarily small $\epsilon>0$.
The $\epsilon$ absorbs the polylogarithmic factors
and the subpolynomial prefactor
$e^{\mathcal{O}(\sqrt{\log T})}=o(T^\epsilon)$.
\end{corollary}

\noindent The assumed form of $\gamma_T$ holds, for example, for the squared-exponential kernel on bounded domains. The two cases reflect a trade-off between computational cost and convergence rate: with constant $M$, the cumulative MC error term $\sqrt{\frac{\log(4/\delta)}{2MT}} = \mathcal{O}(1/\sqrt{T})$ dominates; increasing $M$ tightens the bound but does not change the $T^{-1/2}$ rate. With $M = \Omega(T)$, the MC error becomes $\mathcal{O}(1/T)$, recovering the near-optimal rate. In both cases, the bound vanishes as $T \to \infty$: although the prefactor $e^{\mathcal{O}(\sqrt{\log T})}$ grows without bound, it grows slower than any polynomial, so the polynomial decay in $T$ still dominates.

\paragraph{Average Guarantee}

To complement the high-probability bound on realized terminal MSE,
we provide a posterior-averaged MSE bound with a Bayesian terminal MSE bound.

\begin{theorem}[High-Probability Posterior-Averaged Terminal MSE Bound]\label{thm:avg_mse}
Under Assumptions~\ref{ass:kernel_regularity} and~\ref{ass:target_dist}, let $\mathcal{X} \subset [0,r]^d$ be compact and consider \texttt{AB-SID} or \texttt{TS-SID}  selecting $x_t \in \bar{\mathcal{X}}_t$ (Table~\ref{tab:methods}), where the constraint threshold at each iteration is estimated via $M$ Monte Carlo samples. Then for any $\delta \in (0,1)$, with probability at least $1 - \delta$:
\begin{equation}
\begin{aligned}
\mathbb{E}_{f \vert \mathcal{D}_T}\left[\mathbb{E}_{P_f}\left[(f-\mu_T)^2\right]\right] \leq e^{2\lambda^2}(1+4\lambda^2) e^{2(1+\kappa)|\lambda|(\beta_T^{1/2}+\epsilon_T)} \cdot \left(\frac{C_1\gamma_T}{T} + \epsilon_{\mathrm{MC}}\right),
\end{aligned}
\end{equation}
where $\kappa$, $\beta_T$, $\epsilon_T$, and $\epsilon_{\mathrm{MC}}$ are as in Theorem~\ref{thm:hp_terminal}, and the outer expectation is over terminal posterior $f \vert \mathcal{D}_T$. For finite $\mathcal{X}$ with exact threshold computation, $\epsilon_T = \epsilon_{\mathrm{MC}} = 0$.
\end{theorem}

Here, $\delta$ denotes the failure probability over
the random data-acquisition process, as in
Theorem~\ref{thm:hp_terminal}
As detailed in Appendix~\ref{app:avg_rate}, the average upper bound achieves the same asymptotic rates as Corollary~\ref{cor:rate}, but with tighter logarithmic factors. Specifically, the main term is $\gamma_T/T$ rather than $\beta_T\gamma_T/T$, and the prefactor exponent is also reduced from $2(\kappa+2)|\lambda|$ to $2(\kappa+1)|\lambda|$.

\begin{corollary}[Bayesian Terminal MSE Bound]
\label{cor:bayesian_terminal}
Under the conditions of Theorem~\ref{thm:avg_mse},
for every $\delta\in(0,1)$, the following Bayesian terminal MSE bound holds:
\begin{equation}
\mathbb{E}_{f}\left[
  \mathbb{E}_{\mathcal{D}_T\mid f}
  \left[\mathcal{L}(\mu_T)\right]
\right]
\leq e^{2\lambda^2}(1+4\lambda^2) e^{2(1+\kappa)|\lambda|(\beta_T^{1/2}+\epsilon_T)} \cdot \left(\frac{C_1\gamma_T}{T} + \epsilon_{\mathrm{MC}}\right) +e^{2\lambda^2}(1+4\lambda^2) \delta.
\end{equation}
\end{corollary}
Here, $\delta$ is an arbitrary confidence parameter inherited
from Theorem~\ref{thm:avg_mse}, and the additional term
accounts for the contribution from the event
where its high-probability bound fails. The resulting Bayesian
MSE bound holds unconditionally for any $\delta\in(0,1)$.

\begin{corollary}[Bayesian Terminal MSE Convergence Rate]
\label{cor:bayesian_rate}
Under the conditions of Corollary~\ref{cor:bayesian_terminal},
suppose that $\gamma_T=\mathcal{O}((\log T)^q)$
for some fixed $q\geq 0$.
Then the Bayesian terminal MSE
$\overline{\mathcal{L}}_T :=
\mathbb{E}_f[
\mathbb{E}_{\mathcal{D}_T\mid f}[\mathcal{L}(\mu_T)]]$
satisfies:
(i) $M$ is constant $\Rightarrow$
$\overline{\mathcal{L}}_T=\mathcal{O}(T^{-1/2+\epsilon})$;
(ii) $M=\Omega(T)$ $\Rightarrow$
$\overline{\mathcal{L}}_T=\mathcal{O}(T^{-1+\epsilon})$,
for arbitrarily small $\epsilon>0$.
\end{corollary}

Note that although Corollaries~\ref{cor:rate} and
\ref{cor:bayesian_rate} achieve the same asymptotic rates,
as illustrated in the proof given in Appendix~ \ref{app:avg_rate}, the Bayesian analysis yields a more favorable finite-sample bound, with $\gamma_T/T$ instead of $\beta_T\gamma_T/T$ and a smaller exponential coefficient, up to the additional failure-event term.

\paragraph{Robustness to Boltzmann Parameter Miss-specification.}
In practice the parameters used for data acquisition may not exactly
coincide with those used in test-time evaluation (e.g., in
molecular active learning \citep{kulichenko2023uncertainty}).
We therefore consider bounded misspecification of both the scaling
parameter $\lambda$ and the bias function $b$. As shown in Proposition~\ref{prop:misspecification} in Appendix~\ref{App: Boltzmann_misspecify},
the high-probability guarantee of Theorem~\ref{thm:hp_terminal} can still hold with an additional multiplicative
constant that depends on the degree of parameter misspecification and
the range of $f$ over $\mathcal{X}$.

\textbf{Other Performance Metrics.} While our main focus is $\mathcal{L}(\hat{f})$ defined in Eq.~\eqref{Eq: problem}, the analysis extends to any metric controlled by posterior uncertainty:
\begin{remark}[General Performance Metrics]
\label{rmk:general_metrics}
Suppose a metric $\mathcal{M}_T$ satisfies
\begin{equation*}
    \mathbb{E}_{P_f}[\mathcal{M}_T] \leq C \cdot \mathbb{E}_{P_{\mu_T}}[g(\sigma_T(x))]
\end{equation*}
for some non-decreasing $g:\mathbb{R}_{\geq 0}\to\mathbb{R}_{\geq 0}$ with $g(0)=0$ and $g$ continuous at $0$, and $C>0$. Then $\mathbb{E}_{P_f}[\mathcal{M}_T]\to0$ as $T\to\infty$.
\end{remark}

\begin{figure*}[t]
    \centering
\includegraphics[width=1.0\linewidth]{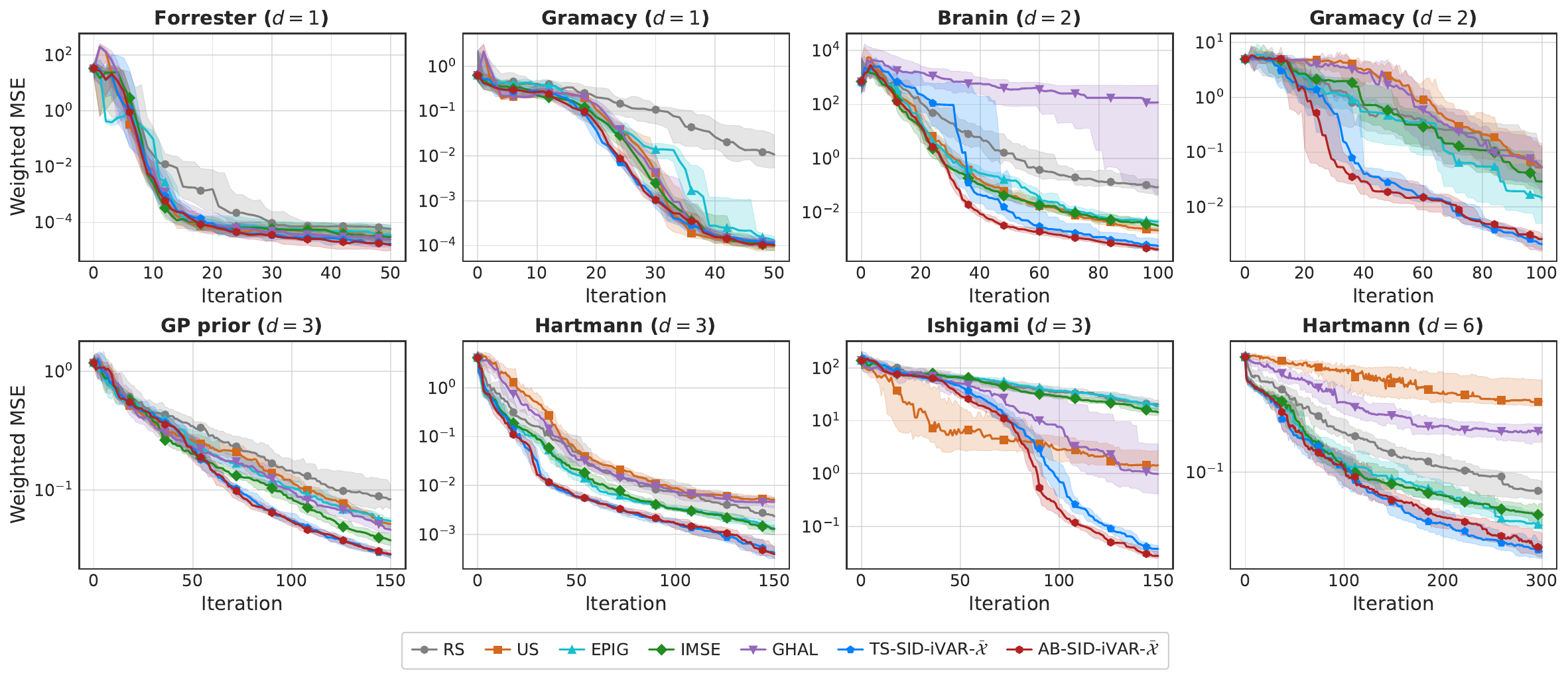}
    \caption{Weighted MSE convergence on synthetic benchmarks across dimensions $d = 1$ to $d = 6$. Each panel shows median performance over 30 random seeds (shaded regions: interquartile range). \texttt{AB-SID-iVAR}-$\bar{\mathcal{X}}$ consistently achieves lower weighted MSE than baselines (RS, US, IMSE), with the gap widening in higher dimensions.}
    \label{fig:synthetic_exp}
\end{figure*}

\subsection*{4.3. Analyzing Existing Heuristics}

The heuristic active learning approaches
mentioned in the Related Works section and detailed in Appendix~\ref{app:hal}, differ in various ad-hoc practical choices, but share a common algorithmic structure (which we term Generic Heuristic Active Learning, \texttt{GHAL}), targeting variance reduction under a Boltzmann surrogate as in our methods. This connection, together with the analytical tools developed in Section~\hyperlink{sec:problem_setting}{3}, allows us to make a first attempt at analyzing their convergence properties.

Specifically, at each iteration, starting from a point outside the feasible region $\bar{\mathcal{X}}_t' := \{x : \sigma^2_{t-1}(x) \geq \eta\}$ with threshold $\eta>0$, a Markov chain is simulated targeting a surrogate distribution of generic form $\tilde{p}_{t-1} \propto \exp\bigl(\lambda\mu_{t-1}(x) + \lambda h(\sigma^2_{t-1}(x)) + b(x)\bigr)$, where $h(\sigma^2)$ is a smooth transformation of $\sigma^2$, until the chain crosses into the feasible region, at which point $x_t$ is queried; if the Markov chain does not cross into the feasible region, the sample of maximum uncertainty is queried.
If we assume that exploration by the Markov chain is exhaustive so that $\sigma^2_{t-1}(x_t) \geq \min\{\eta, \max_{x \in \mathcal{X}} \sigma^2_{t-1}(x)\}$, the conditions of the following convergence guarantee are satisfied:
\begin{proposition}
\label{prop:fixed-threshold}
For a fixed threshold $\eta \in (0, 1]$, define $\bar{\mathcal{X}}'_t := \{x : \sigma^2_{t-1}(x) \geq \eta\}$. Consider an active learning algorithm that selects $x_t$ from the set $\bar{\mathcal{X}}'_t$ if $\bar{\mathcal{X}}'_t \neq \phi$ and $x_t = \arg\max_x \sigma_{t-1}(x)$ otherwise. Then, under Assumptions 2.1 and 3.1, and assuming $\gamma_T = \mathcal{O}((\log T)^q)$ for some fixed $q \geq 0$ as in Corollary~\ref{cor:rate}, for any $\delta \in (0, 1)$, with probability at least $1 - \delta$:
\begin{equation}
\mathcal{L}(\mu_T) = \mathcal{O}(T^{-1+\epsilon})
\end{equation}
for arbitrarily small $\epsilon > 0$ and sufficiently large $T$ such that $\eta > C_1 \gamma_T / T$.
\end{proposition}
To the best of our knowledge, this is the first convergence analysis  under the SIDAL framework for heuristic active learning methods tackling this problem.

\section*{5. Experiments}

We validate our proposed methods against Random Sampling (\texttt{RS}), Uncertainty
Sampling (\texttt{US}), Integrated Mean Squared Error (\texttt{IMSE})~\citep{sacks1989design, seo2000gaussian},
Expected Predictive Information Gain (\texttt{EPIG})~\citep{smith2023prediction}, and
\texttt{GHAL} detailed in Appendix~\ref{app:hal}. We omit comparing with DRAL as it
optimizes for worst-case performance over an ambiguity set rather than exploiting the
known Boltzmann form. We use a Matérn-5/2 GP surrogate by default. For our methods, we
set $\lambda = 1$ and $b(x) = 0$ as default (see Appendix~\ref{app: temp_and_bias} for
sensitivity to other settings). Each experiment starts with a single initial sample
and is repeated 30 times with different random seeds. Full experimental setup is
deferred to Appendix~\ref{App: setup}.

\subsection*{5.1. Synthetic Experiments}

We first evaluate AL performance on standard synthetic benchmarks spanning $d=1$ to $d=6$. Results are illustrated in Figure~\ref{fig:synthetic_exp}. \texttt{AB-SID-iVAR} consistently achieves lower weighted MSE than baselines, with the gap becoming more pronounced in higher dimensions. The Thompson sampling variant \texttt{TS-SID-iVAR} is competitive but exhibits higher variance across seeds, particularly on multimodal targets such as Branin. \texttt{US} shows a clear failure mode: by concentrating queries in high-uncertainty regions regardless of their probability under $P_f$, it can perform worse than random sampling. \texttt{IMSE} is a strong baseline as it also minimizes integrated variance, but learns uniformly and thus underperforms SID-aware approaches. \texttt{GHAL}, while SID-aware and competitive in some tasks,
exhibits high variance and inconsistent performance due to its sensitivity to hyperparameters $\eta$ and $h$; in contrast, our methods avoid task-specific hyperparameter tuning.

\noindent \textbf{Ablation Study.} To understand the contribution of different components,
we conduct an ablation study under two settings: (1) using $\mu_{t-1}$ as a naive surrogate
for SID (i.e., $e^{\lambda \mu_{t-1}}$), tested with both our acquisition
(denoted $\mu$\texttt{-SID-iVAR}) and EPIG (denoted $\mu$\texttt{-SID-EPIG}) to verify
that combining any acquisition function with a plug-in SID surrogate is insufficient,
and (2) removing the constraint set (denoted \texttt{SID-iVAR}-$\phi$). Results are
shown in Figure~\ref{fig:abl_study} and analyzed in detail in Appendix~\ref{App: Ablation_Study}.
While these variants achieve comparable performance on benchmarks dominated by a broad
basin (e.g., Hartmann), both plug-in variants fail in scenarios where the target
distribution has well-separated modes (Branin) or strongly oscillatory structure (Ishigami),
as the point estimate $\mu_{t-1}$ discards posterior uncertainty and commits prematurely
to whichever mode is discovered first. This validates that the Bayesian treatment of $f$
in our surrogate $\tilde{p}_{t-1}$, alongside the constraint set $\bar{\mathcal{X}}_t$,
enables its robust performance.

Finally, we report computational cost comparisons in Appendix~\ref{app: run_time_rep}. In discrete domains, our methods are slightly slower than \texttt{IMSE} ($<$10 seconds per iteration). In continuous domains, our methods incur higher per-iteration cost than \texttt{IMSE} (which uses uniform random samples) due to MCMC sampling and constrained acquisition optimization, but remain practical ($<$30 seconds per iteration), a sensitivity regarding MCMC particle size at Figure~\ref{fig:mcmc_runtime_report} in  Appendix~\ref{app: sensitivity_analysis_of_mcmc} also shows that the empirical performance stays the same while using $N=250d$ particles.

\subsection*{5.2. Potential Energy Surface Modeling}\label{Sec: PES}

Learning the potential energy surface (PES) $E(\mathbf{x})$ of an atomistic
system is a natural SIDAL instance: downstream molecular dynamics samples
configurations from the Boltzmann distribution $P_E(\mathbf{x}) \propto
\exp(-E(\mathbf{x})/k_{\rm B}\mathcal{T})$ (i.e., $\lambda =
-(k_{\rm B}\mathcal{T})^{-1} < 0$, $b \equiv 0$, where $\mathcal{T}>0$ denotes the thermal temperature and $k_{\rm B}$ is the Boltzmann constant), so prediction accuracy is
only required where this self-induced density places mass, typically a small,
a priori unknown fraction of configuration space corresponding to physically
accessible states. Ground-truth queries to $E$ using, e.g., density functional theory \cite{szabo2012modern}, typically cost hours to weeks per evaluation, making
sample efficiency essential.

We benchmark the proposed methods on four PES fitting problems, the PES of (i) a hydrogen molecule on a copper surface  ($\mathrm{H_2}$ on $\mathrm{Cu}$), (ii) a hydrogen atom on a copper cluster ($\mathrm{H_2}$ on $\mathrm{Cu}$ Cluster), (iii) a highly symmetric silicon cluster (Si Crystal), and (iv) a water molecule on a platinum surface ($\mathrm{H_{2}O}$ on $\mathrm{Pt}$). Details, including
coordinate parameterizations and $\lambda$ values, are deferred to
Appendix~\ref{App: PES_modeling_details}.

\begin{figure}[h!]
    \centering
    \includegraphics[width=1.0\linewidth]{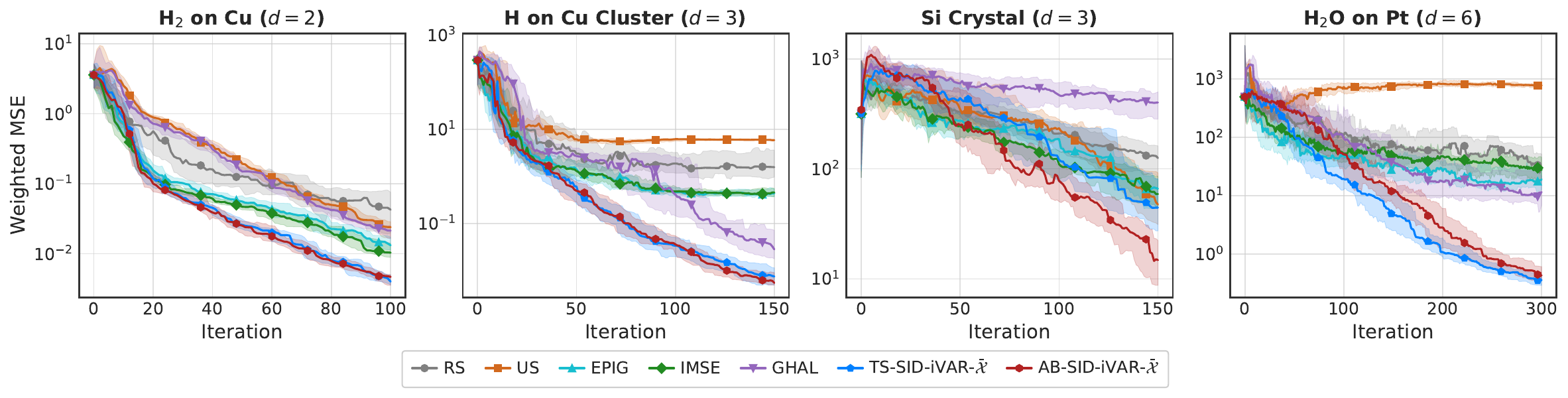}
    \caption{Weighted MSE convergence on four PES fitting tasks of increasing
dimensionality: \mbox{H$_2$ on Cu} ($d{=}2$), \mbox{H on Cu$_{13}$} ($d{=}3$),
\mbox{Si Crystal} ($d{=}3$), and \mbox{H$_2$O on Pt} ($d{=}6$). \texttt{AB-SID-iVAR} consistently achieves the lowest error across
systems.}
\label{fig:MD_Modeling}
    \label{fig:MD_Modeling}
\end{figure}

 The result summarized in Figure~\ref{fig:MD_Modeling} shows that \texttt{AB-SID-iVAR} consistently achieves the lowest weighted MSE, the performance of GHAL, on the other hand, varies between problems, indicating its sensitivity to the specification of hyperparameters.

\subsection*{5.3. Molecular Drug Discovery}\label{Sec: Drug_screening}

\begin{figure*}[t]
    \centering
    \includegraphics[width=1.0\linewidth]{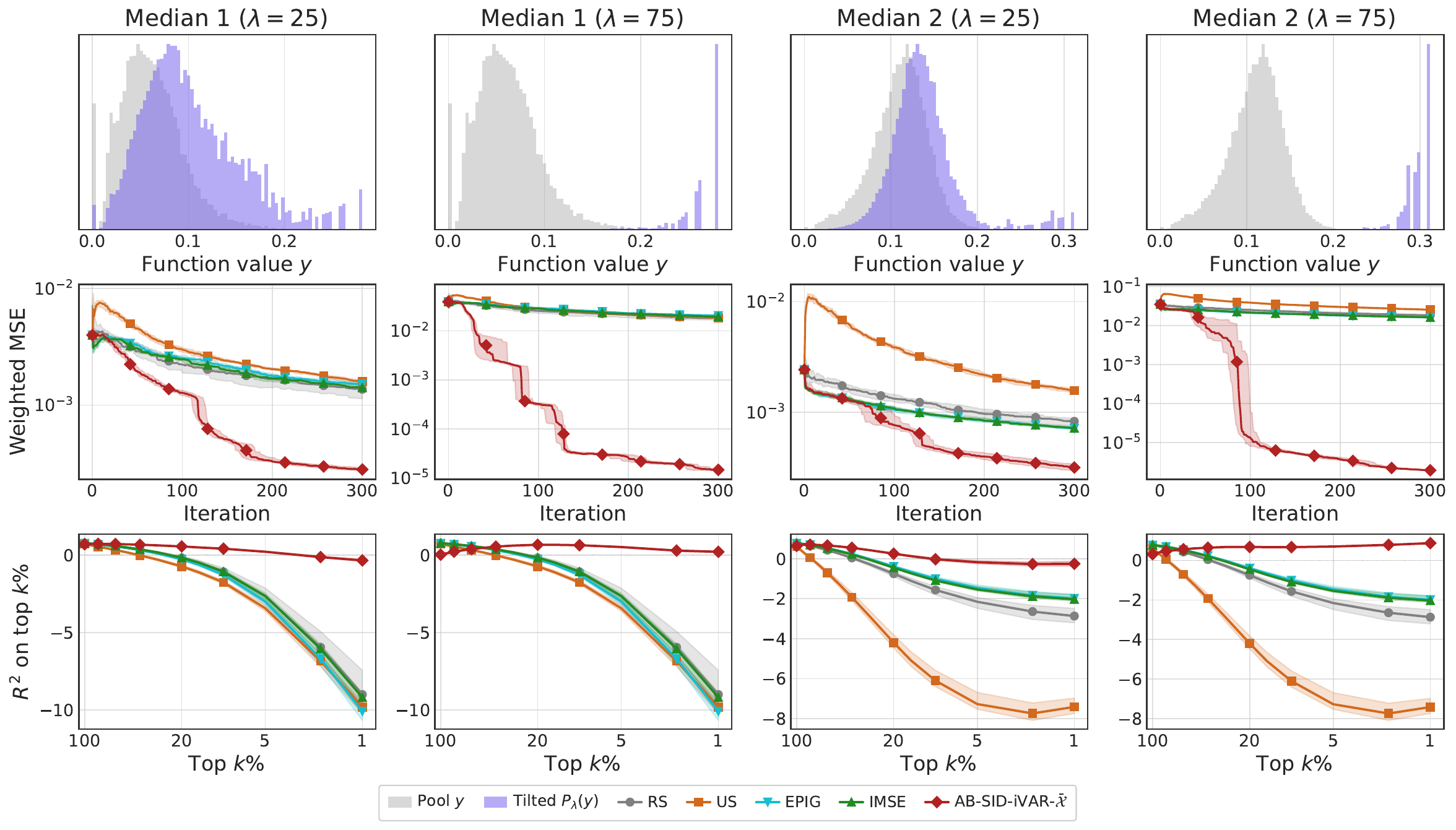}
    \caption{Active learning for molecular property prediction on GuacaMol scoring functions (Median 1, Median 2). The target distribution $P_\lambda(x) \propto \exp(\lambda f(x))$ emphasizes high-scoring compounds, relevant for drug discovery where accurate predictions matter most for promising candidates. \textbf{Top:} pool distribution (gray) vs.\ Boltzmann-tilted density (red). \textbf{Middle:} weighted MSE convergence over 300 iterations; \texttt{AB-SID-iVAR} achieves one to two orders of magnitude lower error than baselines. \textbf{Bottom:} $R^2$ on top $k\%$ compounds ranked by Boltzmann weight; baselines degrade to negative $R^2$ as $k$ shrinks while \texttt{AB-SID-iVAR} remains near zero. Pool contains approximately 20,000 molecules with Tanimoto kernel GP.}
    \label{fig:molecular_drug_discovery}
\end{figure*}

Finally, we evaluate our methods on a discrete pool-based active learning problem using real molecular dataset. In practice, evaluating molecular properties is expensive hence drug discovery campaigns prioritize accurate predictions for promising candidates over uniform accuracy across the library. This naturally matches the SIDAL setting, where the Boltzmann density concentrates on high-scoring compounds that matter most for downstream screening. We use a GP with a Tanimoto kernel over Morgan fingerprints as the surrogate, suited for the discrete molecular domain. This also tests whether our method remains effective beyond the kernel regularity assumed in our theory (Assumption~2.1).

We evaluate on the GuacaMol benchmark~\citep{brown2019guacamol} ($\approx$20{,}000 molecules) using two drug-likeness scoring functions (Median 1, Median 2), with $\lambda \in \{25, 75\}$ \footnote{We note the large $\lambda$ is offset by the narrow score range $f \in [0.1, 0.3]$ (Figure~\ref{fig:molecular_drug_discovery}, top row).} controlling target concentration (Figure~\ref{fig:molecular_drug_discovery}, top row). Setup details are deferred to Appendix~\ref{App: Mol_drug_disc}. Increasing $\lambda$ from 25 to 75 sharpens the target, concentrating mass on the highest-scoring molecules. This exponential tilting appears broadly, including guided sampling in diffusion models~\citep{song2023loss} and reward-tilted preference alignment~\citep{rafailov2023direct}. Based on the convergence curves, \texttt{AB-SID-iVAR} demonstrates strong performance, converging roughly an order of magnitude lower at $\lambda=25$, and this benefit becomes more pronounced at $\lambda = 75$. This reflects the increased cost of learning uniformly when only a narrow slice of the input space is of interest.

Further inspecting final model quality, the bottom row evaluates $R^2$ \footnote{We use  $R^2 = 1 - \sum_i (y_i - \hat{y}_i)^2 / \sum_i (y_i - \bar{y})^2$ computed
on the top-$k\%$ subset, where $\bar{y}$ is the subset mean.} on the top-$k\%$ of molecules ranked by Boltzmann weight. At $k = 100\%$ represent uniform learning objective, all methods achieve comparable $R^2$, reflecting uniform-average performance. As $k$ decreases toward the elite subset, however, SID-unaware baselines degrade to strongly negative $R^2$, indicating predictions worse than a constant baseline on the molecules that matter most. In comparison, \texttt{AB-SID-iVAR} demonstrates significantly better prediction accuracy in high-value regions. The pairwise plots in Figure~\ref{fig:pairwise_plot} confirm this finding: dark (high-weight) points align with the diagonal only for \texttt{AB-SID-iVAR}, while baselines show systematic spread or bias. Together, these results demonstrate that SID-unaware methods can appear reasonable under global metrics yet fail on the subset relevant for downstream screening, and $\lambda$ provides a principled mechanism to specify \emph{which} output region should be learned well, without requiring prior knowledge of where these regions lie in input space.

\section*{6. Conclusions}
We formalized the problem of active learning under self-induced distributions (SIDAL), where the target distribution depends on the unknown function being learned. We proposed two uncertainty reduction acquisition functions based on Gaussian Process surrogates that approximate the intractable Bayesian target in closed form while avoiding partition function estimation. Our theoretical analysis establishes that the terminal prediction error vanishes with high probability under mild conditions, with tighter logarithmic factors available under an average-case analysis; a detailed elaboration of limitations and future work is deferred to Appendix~\ref{App:limit}. Empirically, our methods consistently outperform SID-unaware baselines across synthetic benchmarks and real-world problems, with the gap widening as the target concentrates or the input dimension increases.

\subsection*{Acknowledgments}
This work was supported by Research England under the Expanding Excellence in England (E3) funding stream, awarded to MARS: Mathematics for AI in Real-world Systems in the School of Mathematical Sciences at Lancaster University. The authors also gratefully acknowledge the use of Lancaster University’s High End Computing (HEC) facility, which provided computational resources for this work. The authors are also grateful to \href{https://scholar.google.com/citations?user=JlKKsdAAAAAJ&hl=en&oi=ao}{Yueqi Cao} for helpful discussions on the theoretical analysis presented in this paper, and to \href{https://scholar.google.com/citations?user=LWB9CJYAAAAJ&hl=en&oi=ao}{Zhuoyue Huang} and \href{https://scholar.google.com/citations?user=ieLRNKMAAAAJ&hl=en&oi=ao}{Tom Rainforth} for valuable discussions that helped clarify the positioning of this work within the broader active learning literature.

{
\small

\bibliographystyle{plainnat}
\bibliography{example_paper}
}

\appendix
\onecolumn

\renewcommand{\contentsname}{Appendix Contents}
\tableofcontents

\section{Related Works}\label{App: Related_Work}
Several lines of work consider active learning with structured targets rather than uniform error across the input domain. These methods differ in how the target is specified: a known finite set of test points \citep{yu2006active,hubotter2024transductive}, a continuous target distribution $P$ \citep{kirsch2021test,smith2023prediction}, or an ambiguity set containing possible $P$ when it cannot be exactly specified \citep{Frogner, takeno2025distributionally}, hence optimizing for worst-case performance over the set. Recently, a concurrent work by \citet{huang2026loss} develops a general decision-theoretic framework for loss-driven Bayesian active learning, showing that the inner minimisation over the Bayes-optimal downstream action can be eliminated analytically for weighted Bregman-divergence losses. In particular, they consider exponentially weighted squared-error losses under GP regression, yielding an exponential--Gaussian weighting factor that also appears in our unconstrained AB-SID-iVAR acquisition criterion after the zeroth-order Taylor approximation. However, these expressions arise from different constructions: their exponential weighting is directly specified through the downstream loss, whereas ours is derived by approximating the posterior expectation of a normalised self-induced Boltzmann distribution. The latter involves an intractable expectation over a function-dependent partition function, so the exact Bayes-optimal estimator, although formally characterised in Appendix~\ref{App: Non-Bayes Optimal}, is generally not analytically tractable. We instead adopt the GP posterior mean as a tractable surrogate, introduce an explicit exploration constraint, and establish convergence guarantees for its terminal prediction error under the true self-induced distribution.

Outside the AL literature, related heuristics have emerged in scientific domains \citep{li2015molecular,podryabinkin2017active,vandermause2020fly,jinnouchi2020fly,van2023hyperactive,kulichenko2023uncertainty,duschatko2024uncertainty}, where a surrogate model guides Markov Chain Monte Carlo (MCMC) sampling from a tilted Boltzmann distribution and triggers a query when predictive uncertainty exceeds a pre-specified threshold. These methods often require careful tuning of hyperparameters and lack theoretical guarantees. In Section~\hyperlink{Sec:unification}{4.3}, we formalize a generalized version of such heuristics within our framework and for the first time analyze their convergence properties.

\section{Methodology Details}\label{App: Method_Details}

\subsection{Acquisition Functions} \label{App: acq_func_derivation}
Starting from one-step look-ahead MSE minimization, we make two approximations: (i) replace squared error with posterior variance (according to lemma~\ref{lem:gp_confidence}), and (ii) treat $P_f$ as fixed when evaluating variance reduction, avoiding the need to recompute partition functions for each hypothetical observation. This reduces the acquisition function to minimizing $\mathbb{E}_{f|D_{t-1}}\mathbb{E}_{P_f}[\sigma^2_t \mid x_t = x]$, which requires approximating $\mathbb{E}_{f|D_{t-1}}[P_f(x^\star)]$. Substituting the Boltzmann form:
\begin{equation}
    \mathbb{E}_{f|D_{t-1}}[P_f(x^\star)] = \mathbb{E}_{f|D_{t-1}}\left[\frac{e^{\lambda f(x^\star) + b(x^\star)}}{Z_f}\right].
\label{Eq: expected_var_form}
\end{equation}

\subsubsection{AB-SID} \label{App: AB-SID}
For notation simplicity, let $A = e^{\lambda f(x^\star) + b(x^\star)}$ and $B = Z_f$. In order to compute $\mathbb{E}[A/B]$, We approximate $\mathbb{E}[A/B]$ via Taylor expansion around $B = \mathbb{E}[B]$:
\begin{align*}
\frac{A}{B} = \frac{A}{\mathbb{E}[B]} - \frac{A(B - \mathbb{E}[B])}{(\mathbb{E}[B])^2} + \frac{A(B - \mathbb{E}[B])^2}{(\mathbb{E}[B])^3} - \cdots.
\end{align*}

Taking expectation:
\begin{align*}
\mathbb{E}\left[\frac{A}{B}\right] &= \frac{\mathbb{E}[A]}{\mathbb{E}[B]} - \frac{\mathbb{E}[A(B - \mathbb{E}[B])]}{(\mathbb{E}[B])^2} + \frac{\mathbb{E}[A(B - \mathbb{E}[B])^2]}{(\mathbb{E}[B])^3} - \cdots.
\end{align*}

For the second term, note that:
\begin{equation*}
\mathbb{E}[A(B - \mathbb{E}[B])] = \mathbb{E}[AB] - \mathbb{E}[A]\mathbb{E}[B] = \mathrm{Cov}(A, B),
\end{equation*}
therefore,
\begin{equation*}
\mathbb{E}\left[\frac{A}{B}\right] = \underbrace{\frac{\mathbb{E}[A]}{\mathbb{E}[B]}}_{T_0} - \underbrace{\frac{\mathrm{Cov}(A, B)}{(\mathbb{E}[B])^2}}_{T_1} + \frac{\mathbb{E}[A(B - \mathbb{E}[B])^2]}{(\mathbb{E}[B])^3} - \cdots
\end{equation*}

\textbf{Zero-Order Term $T_0$} For the first term $T_0$, substituting $A$ and $B$'s expression back:
\begin{equation*}
\mathbb{E}_{f|D}[P_f(x^\star)] \approx \frac{\mathbb{E}_{f|D}[e^{\lambda f(x^\star) + b(x^\star)}]}{\mathbb{E}_{f|D}[Z_f]}
\end{equation*}

For the numerator, again using the Gaussian MGF identity, we have:
\begin{equation*}
\mathbb{E}_{f|D}[e^{\lambda f(x^\star) + b(x^\star)}] = \exp\left(\lambda \mu_t(x^\star) + \frac{\lambda^2}{2}\sigma^2_t(x^\star) + b(x^\star)\right)
\end{equation*}

Similarly, for the partition function:
\begin{align*}
\mathbb{E}_{f|D}[Z_f] = \sum_{x'} \exp\left(\lambda \mu_t(x') + \frac{\lambda^2}{2}\sigma^2_t(x') + b(x')\right)
\end{align*}

\textbf{First-Order $T_1$} For the first-order term
\begin{equation*}
T_1 = -\frac{\mathrm{Cov}(A, B)}{(\mathbb{E}[B])^2}
\end{equation*}

For $\mathrm{Cov}(A, B)$, since $B = Z_f = \sum_{x'} e^{\lambda f(x') + b(x')}$, by bi-linearity of covariance:
\begin{equation*}
\mathrm{Cov}(A, B) = \sum_{x'} \mathrm{Cov}\left(e^{\lambda f(x^\star) + b(x^\star)}, e^{\lambda f(x') + b(x')}\right)
\end{equation*}

To evaluate this we first introduce the following helper lemma.
\begin{lemma}[Exponential Covariance under Joint Gaussian]
\label{lem:exp_cov}
Let $(U, V)$ be jointly Gaussian with means $(\mu_U, \mu_V)$, variances $(\sigma_U^2, \sigma_V^2)$, and covariance $\kappa$. Then:
\begin{equation*}
\mathrm{Cov}(e^{\lambda U}, e^{\lambda V}) = \mathbb{E}[e^{\lambda U}]\mathbb{E}[e^{\lambda V}]\left(e^{\lambda^2 \kappa} - 1\right)
\end{equation*}
\end{lemma}

\begin{proof}
By the joint MGF of Gaussian random variables:
\begin{equation*}
\mathbb{E}[e^{\lambda U + \lambda V}] = \exp\left(\lambda(\mu_U + \mu_V) + \frac{\lambda^2}{2}(\sigma_U^2 + \sigma_V^2 + 2\kappa)\right)
\end{equation*}
The product of marginal MGFs:
\begin{equation*}
\mathbb{E}[e^{\lambda U}] \cdot \mathbb{E}[e^{\lambda V}] = \exp\left(\lambda(\mu_U + \mu_V) + \frac{\lambda^2}{2}(\sigma_U^2 + \sigma_V^2)\right)
\end{equation*}
Therefore:
\begin{align*}
\mathrm{Cov}(e^{\lambda U}, e^{\lambda V}) &= \mathbb{E}[e^{\lambda U} \cdot e^{\lambda V}] - \mathbb{E}[e^{\lambda U}]\mathbb{E}[e^{\lambda V}] \nonumber = \mathbb{E}[e^{\lambda U}]\mathbb{E}[e^{\lambda V}]\left(e^{\lambda^2 \kappa} - 1\right)
\end{align*}
\end{proof}

Applying Lemma~\ref{lem:exp_cov} with $U = f(x^\star) + \frac{b(x^\star)}{\lambda}$, $V = f(x') + \frac{b(x')}{\lambda}$, and hence $\kappa = k_{t-1}(x^\star, x')$:
\begin{equation*}
\begin{aligned}
& \mathrm{Cov}\left(e^{\lambda f(x^\star) + b(x^\star)}, e^{\lambda f(x') + b(x')}\right) = \mathbb{E}\left[e^{\lambda f(x^\star) + b(x^\star)}\right]\mathbb{E}\left[e^{\lambda f(x') + b(x')}\right]\left(e^{\lambda^2 k_{t-1}(x^\star, x')} - 1\right)
\end{aligned}
\end{equation*}
Summing over $x'$:
\begin{equation*}
\mathrm{Cov}(A, B) = \mathbb{E}[A] \sum_{x'} \mathbb{E}\left[e^{\lambda f(x') + b(x')}\right]\left(e^{\lambda^2 k_{t-1}(x^\star, x')} - 1\right)
\end{equation*}
Since $\mathbb{E}[B] = \sum_{x'} \mathbb{E}\left[e^{\lambda f(x') + b(x')}\right]$, the first-order term becomes:
\begin{align*}
T_1 &= -\frac{\mathbb{E}[A]}{(\mathbb{E}[B])^2} \sum_{x'} \mathbb{E}\left[e^{\lambda f(x') + b(x')}\right]\left(e^{\lambda^2 k_{t-1}(x^\star, x')} - 1\right) \nonumber \\
&= -\frac{\mathbb{E}[A]}{\mathbb{E}[B]} \sum_{x'} \frac{\mathbb{E}\left[e^{\lambda f(x') + b(x')}\right]}{\mathbb{E}[B]}\left(e^{\lambda^2 k_{t-1}(x^\star, x')} - 1\right) \nonumber \\
&= -\frac{\mathbb{E}[A]}{\mathbb{E}[B]} \cdot \frac{\sum_{x'} \tilde{p}^{\mathrm{u}}_{t-1}(x')\left(e^{\lambda^2 k_{t-1}(x^\star, x')} - 1\right)}{\sum_{x''} \tilde{p}^{\mathrm{u}}_{t-1}(x'')}
\end{align*}
where $\tilde{p}^{\mathrm{u}}_{t-1}(x') = e^{\left(\lambda \mu_{t-1}(x') + \frac{\lambda^2}{2}\sigma_{t-1}^2(x') + b(x')\right)}$, the same unnormalized surrogate as in Table~\ref{tab:methods}.

\subsubsection{TS-SID}
Another approach is to approximate Eq.~\ref{Eq: expected_var_form} by a Monte Carlo estimate,

\begin{equation*}
    \mathbb{E}_{f|D_{t-1}}[P_f(x^\star)] \approx \frac{1}{N} \sum_{i=1}^N\left[\frac{e^{\lambda f_i(x^\star) + b(x^\star)}}{Z_{f_i}}\right].
\end{equation*}

Specifically we use $N=1$, which corresponds to Thompson sampling.

\subsubsection{Acquisition Illustrations}

\begin{figure}[h!]
    \centering
    \includegraphics[width=0.9\linewidth]{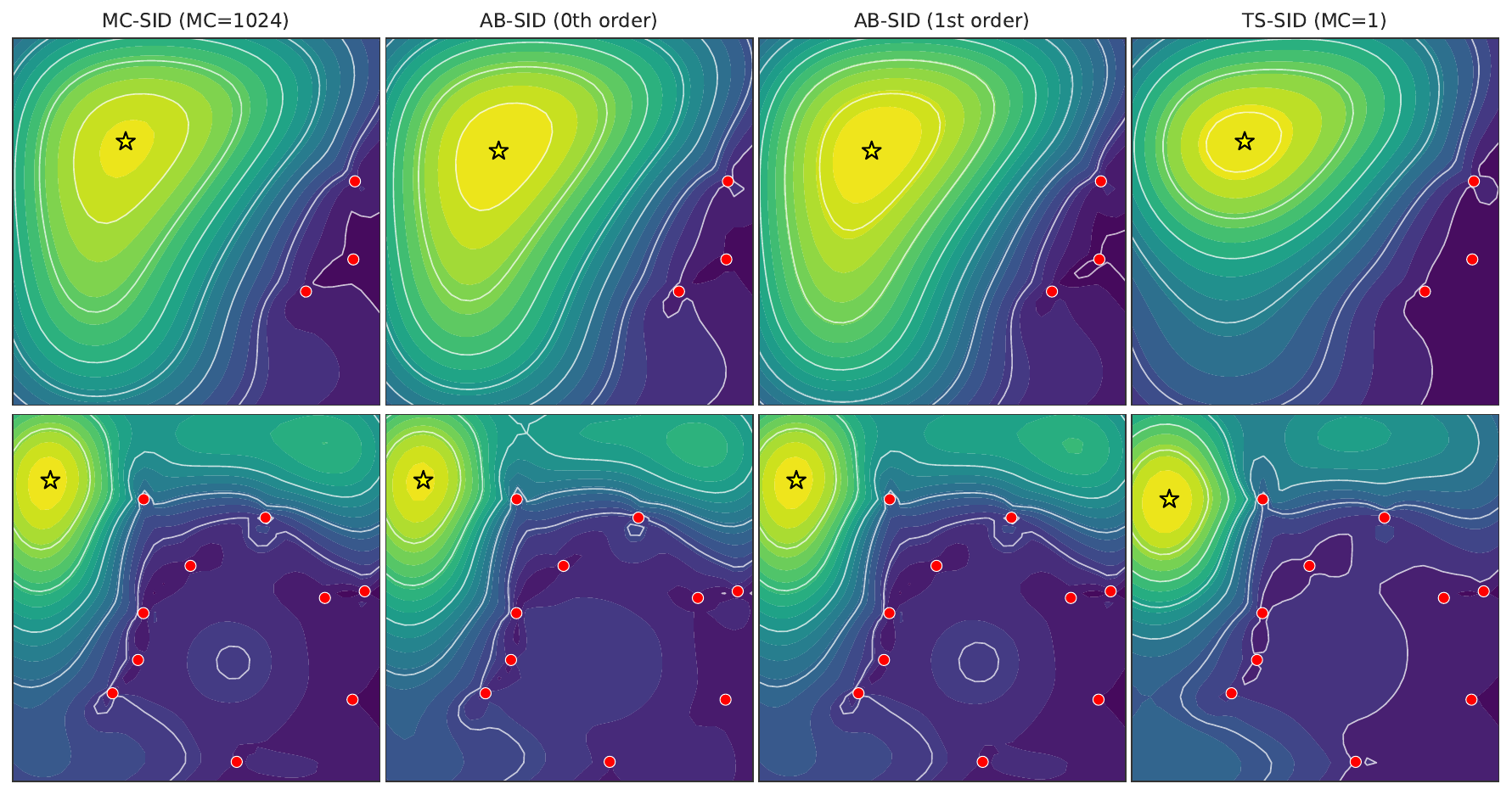}
   \caption{Acquisition function comparison on a 2D GP prior. We use Monte Carlo with 1024 samples (MC-SID, leftmost) as a reference approximation of Eq.~\ref{Eq: expected_var_form}. The zero-order \texttt{AB-SID}approximation closely matches the reference contours regardless of training data size ({\color{red}$\bullet$}), while the first-order correction provides marginal improvement especially when data is abundant. \texttt{TS-SID}  (single posterior sample) shows higher variability. Stars (\yellowstar) indicate the acquisition maximum.}
    \label{fig:acq_approx_comparison}
\end{figure}
In order to better understand the approximation efficacy of the zero order and first order approximation of Eq.~\ref{Eq: expected_var_form}, we provide a two dimensional visual comparison of different acquisition function variants in Figure~\ref{fig:acq_approx_comparison}, which helps justifying the usage of our zero order approximation.

\subsubsection{MCMC approximation of $\tilde{p}_t$ \label{App: MCMC_sample}}
We leverage Sequential Monte Carlo (SMC) \cite{del2006sequential} as our importance sampler for $P_f$, using annealed tempering
  with Random Walk Metropolis (RWM) rejuvenation steps. The temperature schedule is adaptively determined via bisection to maintain a
   target effective sample size (ESS) ratio \cite{jasra2011inference, schafer2013sequential}.

\begin{algorithm}[h]                                    \caption{Sequential Monte Carlo (SMC) Sampling}        \label{alg:smc}                                        \begin{algorithmic}[1]                                 \REQUIRE Log unnormalized density $\log \tilde{p}(x)$, bounds $[l, u]$, particles $N$, ESS threshold $\alpha$, RWM steps $K$
  \ENSURE Weighted samples $\{(x_i, w_i)\}_{i=1}^N$ approximating $p(x) \propto \tilde{p}(x)$            \STATE $x_i \sim \text{Uniform}(l, u)$ for $i = 1, \ldots, N$ \cmt{// Initialize particles}             \STATE $\log w_i \gets 0$ for all $i$ \cmt{// Uniform weights}                                      \STATE $\tau \gets 0$ \cmt{// Initial temperature}  \WHILE{$\tau < 1$}                                   \STATE $\tau_{\text{prev}} \gets \tau$               \STATE $\tau \gets \min(\tau', 1)$ where $\tau'$ satisfies $\text{ESS}(\tau') \approx \alpha \cdot N$ \cmt{// Adaptive
  tempering}                                            \STATE $\log w_i \gets \log w_i + (\tau - \tau_{\text{prev}}) \cdot \log \tilde{p}(x_i)$ for all $i$ \cmt{// Reweight}
      \IF{$\text{ESS} < \alpha \cdot N$}               \STATE $\{x_i\} \gets \text{SystematicResample}(\{x_i\}, \{\log w_i\})$ \cmt{// Resample}           \STATE $\log w_i \gets 0$ for all $i$          \ENDIF                                           \FOR{$k = 1, \ldots, K$}                          \STATE $x_i \gets \text{RWM}(x_i, \tau \cdot \log \tilde{p})$ for all $i$ \cmt{// Rejuvenate} \ENDFOR                                          \ENDWHILE                                         \STATE \textbf{return} $\{(x_i, \exp(\log w_i))\}_{i=1}^N$                                 \end{algorithmic}                                \end{algorithm}

\section{Theoretic Analysis}
\subsection{Non-Bayes Optimality} \label{App: Non-Bayes Optimal}
We note that the Bayes-optimal estimator minimizes the posterior expected loss:
$$R(\hat{f}) = \mathbb{E}_{f|D_t}\left[\sum_{x \in \mathcal{X}} P_f(x) \cdot (\hat{f}(x) - f(x))^2\right]$$
Taking the derivative with respect to $\hat{f}$:
$$\frac{dR}{d\hat{f}} = \mathbb{E}_{f|D_t}[2P_f \cdot (\hat{f} - f)] = 0$$
This yields:
$$\hat{f}^\star(x') = \frac{\mathbb{E}_{f|D_t}[P_f(x') \cdot f(x')]}{\mathbb{E}_{f|D_t}[P_f(x')]} \neq \mu_t(x')$$
Since $P_f(x') = g(f(x'))/Z_f$ depends on $f$ through both $g(f(x'))$ and the partition function $Z_f = \sum_x g(f(x))$, the numerator and denominator involve the intractable expectation. Hence we use $\mu_t(x)$ as a computationally tractable approximation.

\subsection{Lemma~\ref{lem:log_density_ratio}}

\begin{proof}
Expanding the log-density difference:
\begin{equation*}
    \log P_{g_1}(x) - \log P_{g_2}(x)  = \lambda(g_1(x) - g_2(x)) + \log Z_{g_2} - \log Z_{g_1}
\end{equation*}
where $Z_{g_i} = \int_{x'} e^{\lambda g_i(x') + b(x')} dx'$. The first term by definition is:
\begin{equation*}
    |\lambda(g_1(x) - g_2(x))|  = |\lambda| \Delta(x)
\end{equation*}
For the partition function ratio:
\begin{equation*}
    Z_{g_1} = \int_{x'} e^{\lambda g_2(x') + b(x')} e^{\lambda(g_1(x') - g_2(x'))} dx'
\end{equation*}
Since $|g_1(x') - g_2(x')| \leq \Delta_{\max}$ for all $x'$ by definition:
\begin{equation*}
    e^{-|\lambda| \Delta_{\max}} Z_{g_2} \leq Z_{g_1} \leq e^{|\lambda| \Delta_{\max}} Z_{g_2}
\end{equation*}
Taking logarithms:
\begin{equation}
     |\log Z_{g_1} - \log Z_{g_2}| \leq |\lambda| \Delta_{\max}
\label{Eq: log_partition_inequality}
\end{equation}
Combining via triangle inequality:
\begin{equation*}
    |\log P_{g_1}(x) - \log P_{g_2}(x)|  \leq   |\lambda| \Delta(x) + |\lambda| \Delta_{\max}
\end{equation*}
\end{proof}

\subsection{Lemma~\ref{lem:hp_mse_bound}}

\begin{proof}
By Lemma~\ref{lem:gp_confidence}, with probability at least $1-\delta$, at iteration $T$ we have:
\begin{equation*}
    |f(x) - \mu_T(x)| \leq \beta_T^{1/2}\sigma_T(x) + \epsilon_T, \quad \forall x \in \mathcal{X}
\end{equation*}
Applying Lemma~\ref{lem:log_density_ratio} with $(g_1, g_2) = (f, \mu_T)$ we have
$\Delta(x) = |f(x) - \mu_T(x)| \leq \beta_T^{1/2}\sigma_T(x) + \epsilon_T$ and
$\Delta_{\max} \leq \beta_T^{1/2}\sigma_{\max,T} + \epsilon_T$. Additionally with log partition function difference inequality (Eq.~\ref{Eq: log_partition_inequality}):
\begin{align*}
    \log \frac{P_f(x)}{P_{\mu_T}(x)} 
    &\leq |\lambda||f(x) - \mu_T(x)| + |\lambda|\Delta_{\max} \\
    &\leq |\lambda|(\beta_T^{1/2}\sigma_T(x) + \epsilon_T) + |\lambda|(\beta_T^{1/2}\sigma_{\max,T} + \epsilon_T)
\end{align*}
Using $\sigma_T(x) \leq \sigma_{\max,T}$ and exponentiating:
\begin{equation*}
    \frac{P_f(x)}{P_{\mu_T}(x)} \leq e^{2|\lambda|(\beta_T^{1/2}\sigma_{\max,T} + \epsilon_T)}
\end{equation*}

Bounding the MSE:

\begin{align*}
    \mathbb{E}_{P_f}[(f - \mu_T)^2] &= \mathbb{E}_{P_{\mu_T}}\left[\frac{P_f(x)}{P_{\mu_T}(x)} (f(x) - \mu_T(x))^2\right]
    \\
    &\leq e^{2|\lambda|(\beta_T^{1/2}\sigma_{\max,T} + \epsilon_T)} \cdot \mathbb{E}_{P_{\mu_T}}\left[(\beta_T^{1/2}\sigma_T(x) + \epsilon_T)^2\right] \\
    &= e^{2|\lambda|(\beta_T^{1/2}\sigma_{\max,T} + \epsilon_T)} \cdot \mathbb{E}_{P_{\mu_T}}\left[\beta_T \sigma_T^2(x) + 2\beta_T^{1/2}\sigma_T(x)\epsilon_T + \epsilon_T^2\right] \\
    &\leq e^{2|\lambda|(\beta_T^{1/2}\sigma_{\max,T} + \epsilon_T)} \cdot \left(\beta_T \cdot \mathbb{E}_{P_{\mu_T}}\left[\sigma_T^2(x)\right] + 2\beta_T^{1/2}\epsilon_T + \epsilon_T^2\right)
\end{align*}
where the last inequality holds as $\sigma_T \leq 1$.

\end{proof}

\subsection{Lemma~\ref{lem:bayesian_mse}}

We first establish a helper lemma for expectations involving Gaussian random variables.

\begin{lemma}\label{lem:gaussian_mgf}
Let $(\xi, \zeta)$ be jointly Gaussian with $\mathbb{E}[\xi] = \mathbb{E}[\zeta] = 0$, $\mathbb{V}(\xi) = u^2$, $\mathbb{V}(\zeta) = w^2$, and $\mathrm{Cov}(\xi, \zeta) = c$. Then:
\begin{equation*}
    \mathbb{E}[\xi^2 e^{\lambda\zeta}] = (u^2 + \lambda^2 c^2) e^{\lambda^2 w^2 / 2}
\end{equation*}
\end{lemma}

\begin{proof}
Using the conditional distribution $\xi|\zeta \sim \mathcal{N}\left(\frac{c}{w^2}\zeta, u^2 - \frac{c^2}{w^2}\right)$:
\begin{equation*}
    \mathbb{E}[\xi^2|\zeta] = \mathbb{V}[\xi|\zeta] + \left(\mathbb{E}[\xi|\zeta] \right)^2 = u^2 - \frac{c^2}{w^2} + \frac{c^2}{w^4}\zeta^2
\end{equation*}
By the tower property:
\begin{align*}
    \mathbb{E}[\xi^2 e^{\lambda\zeta}] &= \mathbb{E}\left[\mathbb{E}[\xi^2 e^{\lambda\zeta}|\zeta]\right] = \mathbb{E}\left[e^{\lambda\zeta} \mathbb{E}[\xi^2|\zeta]\right] = \mathbb{E}\left[e^{\lambda\zeta} \left(u^2 - \frac{c^2}{w^2} + \frac{c^2}{w^4}\zeta^2\right)\right] \\ & = \left(u^2 - \frac{c^2}{w^2}\right)\mathbb{E}[e^{\lambda\zeta}] + \frac{c^2}{w^4}\mathbb{E}[\zeta^2 e^{\lambda\zeta}]
\end{align*}
Using the Gaussian MGF, $\mathbb{E}[e^{\lambda\zeta}] = e^{\lambda^2 w^2/2}$ and $\mathbb{E}[\zeta^2 e^{\lambda\zeta}] = (w^2 + \lambda^2 w^4)e^{\lambda^2 w^2/2}$ $\left(\mathrm{i.e.,}\ \frac{d\mathbb{E}[e^{\lambda\zeta}]}{d\lambda^2}\right)$, the result follows.
\end{proof}

Now we prove Lemma~\ref{lem:bayesian_mse}.

\begin{proof}

Let $r_t(x) = f(x) - \mu_{t-1}(x)$ denote the residual. Under the GP posterior, $r_t(x) \sim \mathcal{N}(0, \sigma_{t-1}^2(x))$. Using $f(x) = \mu_{t-1}(x) + r_t(x)$ and noting that $Z_f  = Z_{\mu_{t-1}} \cdot \mathbb{E}_{P_{\mu_{t-1}}}[e^{\lambda r_t(x')}]$, the weighted MSE becomes:
\begin{align*}
    L_t := \mathbb{E}_{P_f}[r_t(x)^2] &= \int \frac{e^{\lambda \mu_{t-1}(x) + b(x)} \cdot e^{\lambda r_t(x)}}{Z_{\mu_{t-1}} \cdot \mathbb{E}_{P_{\mu_{t-1}}}[e^{\lambda r_t(x')}]} r_t(x)^2 dx = \frac{\mathbb{E}_{P_{\mu_{t-1}}}[e^{\lambda r_t(x)} r_t(x)^2]}{\mathbb{E}_{P_{\mu_{t-1}}}[e^{\lambda r_t(x')}]}
\end{align*}

Applying Jensen's inequality to the denominator (using convexity of $e^x$):
\begin{equation*}
    L_t \leq \mathbb{E}_{P_{\mu_{t-1}}}[e^{\lambda r_t(x)} r_t(x)^2] \cdot e^{-\lambda\mathbb{E}_{P_{\mu_{t-1}}}[r_t(x')]}
\end{equation*}

Define $\bar{r}_t := \mathbb{E}_{P_{\mu_{t-1}}}[r_t(x')]$. Since $e^{-\lambda \bar{r}_t}$ does not depend on $x$, we have:
\begin{align*}
    L_t &\leq \mathbb{E}_{P_{\mu_{t-1}}}[e^{\lambda r_t(x)} r_t(x)^2] \cdot e^{-\lambda \bar{r}_t} = \mathbb{E}_{P_{\mu_{t-1}}}[r_t(x)^2 e^{\lambda(r_t(x) - \bar{r}_t)}]
\end{align*}

Taking the posterior expectation over $f|D_{t-1}$ and applying Fubini's theorem:
\begin{align*}
    \mathbb{E}_{f|D_{t-1}}[L_t] &\leq  \mathbb{E}_{P_{\mu_{t-1}}}\left[\mathbb{E}_{f|D_{t-1}}\left[r_t(x)^2 e^{\lambda(r_t(x) - \bar{r}_t)}\right]\right]
\end{align*}

Define $\zeta(x) := r_t(x) - \bar{r}_t$. Under the posterior, $(r_t(x), \zeta(x))$ are jointly Gaussian with:

\begin{align*}
    \mathbb{V}(\zeta(x)) &= \sigma_{t-1}^2(x) + V_{t-1}^2 - 2C_{t-1}(x)\\
    \mathrm{Cov}(r_t(x), \zeta(x)) &  = \sigma_{t-1}^2(x) - C_{t-1}(x)
\end{align*}

where
\begin{align*}
    V_{t-1}^2 :& = \mathbb{E}_{f|D_{t-1}}\left[\left(\mathbb{E}_{x' \sim P_{\mu_{t-1}}}[r_t(x')]\right)^2\right] = \mathbb{E}_{f|D_{t-1}}\left[\mathbb{E}_{(x',x'')\sim P_{\mu_{t-1}}^2}[r_t(x') r_t(x'')]\right]
    \\ &  = \mathbb{E}_{(x',x'')\sim P_{\mu_{t-1}}^2}[k_{t-1}(x',x'')]
\end{align*}
and (similarly)
\begin{align*}
    C_{t-1}(x) &:= \mathrm{Cov}(r_t(x), \bar{r}_t) = \mathbb{E}_{f|D_{t-1}}[r_t(x) \bar{r}_t] = \mathbb{E}_{f|D_{t-1}}\left[r_t(x) \mathbb{E}_{x' \sim P_{\mu_{t-1}}}[r_t(x')]\right] \\
    & = \mathbb{E}_{x'\sim P_{\mu_{t-1}}}[k_{t-1}(x,x')]
\end{align*}
using $\mathbb{E}_{f|D_{t-1}}[r_t(x)] = 0$ and $\mathbb{E}_{f|D_{t-1}}[r_t(x) r_t(x')] = k_{t-1}(x, x')$ and Fubini's theorem.

By Cauchy-Schwarz inequality on the positive semi-definite kernel (i.e., $|k_{t-1}(x,x')| \leq \sigma_{t-1}(x)\sigma_{t-1}(x') \leq \sigma_{t-1}(x)$), we have $|C_{t-1}(x)| \leq \sigma_{t-1}(x)$, so:
\begin{align*}
    & |\sigma_{t-1}^2(x) - C_{t-1}(x)| \leq |\sigma_{t-1}^2(x)| + |C_{t-1}(x)| \leq 2\sigma_{t-1}(x) \\
    & \mathbb{V}(\zeta(x)) \leq   4
\end{align*}

Applying Lemma~\ref{lem:gaussian_mgf} with $\xi = r_t(x)$ and $\zeta = \zeta(x)$ as defined above, we have $u^2 = \sigma_{t-1}^2(x)$, $w^2 = \mathbb{V}(\zeta(x))$, and $c = \sigma_{t-1}^2(x) - C_{t-1}(x)$:
\begin{align*}
    \mathbb{E}_{f|D_{t-1}}[r_t(x)^2 e^{\lambda\zeta(x)}] &= (\sigma_{t-1}^2(x) + \lambda^2(\sigma_{t-1}^2(x) - C_{t-1}(x))^2) e^{\lambda^2 \mathbb{V}(\zeta(x))/2} \\
    &\leq (\sigma_{t-1}^2(x) + 4\lambda^2\sigma_{t-1}^2(x)) e^{2\lambda^2} \\
    &= (1 + 4\lambda^2)\sigma_{t-1}^2(x) \cdot e^{2\lambda^2}
\end{align*}
Finally, taking expectation over $x \sim P_{\mu_{t-1}}$:
\begin{equation*}
    \mathbb{E}_{f|D_{t-1}}[L_t] \leq e^{2\lambda^2}(1 + 4\lambda^2) \cdot \mathbb{E}_{P_{\mu_{t-1}}}[\sigma_{t-1}^2(x)]
\end{equation*}
\end{proof}

\subsection{Theorem~\ref{thm:hp_terminal} and Corollary~\ref{cor:rate}} \label{App: them_main_hp_finit_hp}

We first recall a standard result on information gain.

\begin{lemma}[Information Gain Bound {\citep[Lemma 5.4]{srinivas2009gaussian}}]\label{lem:info_gain}
For any sequence of GP queries $x_1, \ldots, x_T$:
\begin{equation*}
\sum_{t=1}^T \sigma_{t-1}^2(x_t) \leq C_1 \gamma_T,
\end{equation*}
where $C_1 = 2/\log(1+\tau^{-2})$ and $\tau^2$ is the noise variance.
\end{lemma}

\paragraph{Uniform confidence event.} Throughout this subsection we work on the event that the bound of Lemma~\ref{lem:gp_confidence} holds simultaneously for all residuals $f - \mu_{t-1}$, $t \le T+1$. By a union bound over these $T+1$ rounds, using a common discretizations grid of spacing $1/(dT)$, this event has probability at least $1 - \delta_{\mathrm{GP}}$ when Lemma~\ref{lem:gp_confidence} is invoked with failure probability $\delta_{\mathrm{GP}}/(T+1)$. As this only adds a $\log(T+1)$ term inside $\beta_T$ and $L$, we retain the notation $\beta_T, \epsilon_T$, with $\beta_T = \mathcal{O}(\log T)$ and $\epsilon_T = \mathcal{O}(\sqrt{\log T}/T)$. In Theorems~\ref{thm:hp_terminal} and~\ref{thm:avg_mse}, $\beta_T, \epsilon_T$ are the coefficients of this event with total GP budget $\delta_{\mathrm{GP}} = \delta/2$ (split equally between the events for $f$ and for the posterior samples in the case of \texttt{TS-SID}). Lemma~\ref{lem:hp_mse_bound} is applied on the same event.

\begin{lemma}[Variance Bound for \texttt{AB-SID-iVAR}]\label{lem:var_bound_ab}
Consider \texttt{AB-SID-iVAR} selecting $x_t \in \bar{\mathcal{X}}_t$ where $\bar{\mathcal{X}}_t$ is defined in Table~\ref{tab:methods}. If the constraint threshold at each iteration is estimated via $M$ Monte Carlo samples, then for any $\delta_{\mathrm{GP}}, \delta_{\mathrm{MC}} \in (0,1)$, with probability at least $1 - \delta_{\mathrm{GP}} - \delta_{\mathrm{MC}}$:
\begin{equation*}
\mathbb{E}_{P_{\mu_T}}[\sigma_T^2] \leq e^{4|\lambda|(\beta^{1/2}_T + \epsilon_T)} \cdot \left(\frac{C_1\gamma_T}{T} + \frac{2\log(2/\delta_{\mathrm{MC}})}{3T} + \sqrt{\frac{\log(2/\delta_{\mathrm{MC}})}{2MT}}\right),
\end{equation*}
where $\beta_T, \epsilon_T$ are as in Lemma~\ref{lem:gp_confidence} with failure probability $\delta_{\mathrm{GP}}$. For finite $\mathcal{X}$ with exact threshold computation, the $\log(2/\delta_{\mathrm{MC}})$ terms vanish.
\end{lemma}

\begin{proof}
The proof proceeds in four parts: (1) relating the surrogate distribution to the target, (2) establishing a martingale structure for MC errors, (3) bounding the cumulative variance, and (4) handling the distribution shift.

Since $\tilde{p}^{\mathrm{u}}_{t-1} = e^{\lambda\mu_{t-1} + \frac{\lambda^2}{2}\sigma_{t-1}^2 + b}$ and its normalized version satisfies $\tilde{p}_{t-1}(x) \propto P_{\mu_{t-1}}(x) \cdot e^{\frac{\lambda^2}{2}\sigma_{t-1}^2}$, we have:
\begin{equation}
\begin{aligned}
\mathbb{E}_{\tilde{p}_{t-1}}[\sigma_{t-1}^2] & = \frac{\mathbb{E}_{P_{\mu_{t-1}}}[\sigma_{t-1}^2 \cdot e^{\frac{\lambda^2}{2}\sigma_{t-1}^2}]}{\mathbb{E}_{P_{\mu_{t-1}}}[e^{\frac{\lambda^2}{2}\sigma_{t-1}^2}]}   \geq \mathbb{E}_{P_{\mu_{t-1}}}[\sigma_{t-1}^2],
\end{aligned}
\label{Eq: inequality_between_surrogate_to_target_integrated_var}
\end{equation}
\noindent where the last inequality uses $\mathrm{Cov}_{P_{\mu_{t-1}}}(\sigma_{t-1}^2, e^{\frac{\lambda^2}{2}\sigma_{t-1}^2}) \geq 0$: let $x, x' \overset{\text{i.i.d.}}{\sim} P_{\mu_{t-1}}$, $h(x) = \sigma_{t-1}^2(x)$, $w(x) = e^{\frac{\lambda^2}{2}\sigma_{t-1}^2(x)}$. Since $w$ is increasing in $h$, $(h(x) - h(x'))$ and $(w(x) - w(x'))$ have the same sign, so $\mathrm{Cov}(h, w) = \frac{1}{2}\mathbb{E}[(h(x) - h(x'))(w(x) - w(x'))] \geq 0$.

The MC estimate of the $\bar{\mathcal{X}}$'s threshold at round $t$:
\begin{equation}
\hat{\mathbb{E}}_{\tilde{p}_{t-1}}[\sigma_{t-1}^2] := \frac{1}{M}\sum_{i=1}^{M} \sigma_{t-1}^2(x^{(i)}), \quad x^{(i)} \overset{\text{i.i.d.}}{\sim} \tilde{p}_{t-1}.
\end{equation}

Let $\mathcal{F}_{t-1}$ denote the $\sigma$-algebra generated by $D_{t-1} = \{(x_i, y_i)\}_{i=1}^{t-1}$ together with all Monte Carlo samples used up to round $t-1$. Since $\tilde{p}_{t-1}$ is fully determined by $\mu_{t-1}$ and $\sigma_{t-1}^2$, which are $\mathcal{F}_{t-1}$-measurable, and the round-$t$ samples $x^{(i)}$ are drawn afresh, the estimation error
\begin{equation}
\xi_t := \hat{\mathbb{E}}_{\tilde{p}_{t-1}}[\sigma_{t-1}^2] - \mathbb{E}_{\tilde{p}_{t-1}}[\sigma_{t-1}^2] = \frac{1}{M}\sum_{i=1}^{M} \sigma_{t-1}^2(x^{(i)}) - \mathbb{E}_{\tilde{p}_{t-1}}[\sigma_{t-1}^2]
\end{equation}
satisfies $\mathbb{E}[\xi_t | \mathcal{F}_{t-1}] = 0$, forming a martingale difference sequence.

Since $\sigma_{t-1}^2(x) \in [0, 1]$ by kernel normalization, the increments are bounded $|\xi_t| \leq 1$. For the conditional variance, given $\mathcal{F}_{t-1}$, the expectation term is a constant and the samples $x^{(i)} \overset{\text{i.i.d.}}{\sim} \tilde{p}_{t-1}$, so by the sample mean variance formula:
$$\mathrm{Var}(\xi_t | \mathcal{F}_{t-1}) = \mathrm{Var}\left(\frac{1}{M}\sum_{i=1}^{M} \sigma_{t-1}^2(x^{(i)}) \Big| \mathcal{F}_{t-1}\right) = \frac{1}{M}\mathrm{Var}_{x \sim \tilde{p}_{t-1}}(\sigma_{t-1}^2(x)) \leq \frac{1}{4M},$$
where the last inequality uses $\mathrm{Var}(Y) \leq 1/4$ for any random variable $Y \in [0,1]$. Hence, the cumulative error $S_T := \sum_{t=1}^T \xi_t$ has total conditional variance $V_T := \sum_{t=1}^T \mathrm{Var}(\xi_t | \mathcal{F}_{t-1}) \leq \frac{T}{4M}$.

By Freedman's inequality \citep{freedman1975tail} (and abuse of notation of $\epsilon$), for any $\epsilon > 0$:
\begin{equation*}
\Pr\left(|S_T| \geq \epsilon\right) \leq 2\exp\left(-\frac{\epsilon^2}{2(V_T + \epsilon/3)}\right) \leq 2\exp\left(-\frac{\epsilon^2}{2(T/(4M) + \epsilon/3)}\right).
\end{equation*}

To obtain a high-probability bound, we seek $\epsilon$ such that the r.h.s. is at most $\delta_{MC}$. This requires:
\begin{equation*}
\frac{\epsilon^2}{2(T/(4M) + \epsilon/3)} \geq \log(2/\delta_{MC}).
\end{equation*}

Rearranging yields the quadratic inequality, we have:
\begin{equation*}
\epsilon^2 - \frac{2\log(2/\delta_{MC})}{3}\epsilon - \frac{T\log(2/\delta_{MC})}{2M} \geq 0.
\end{equation*}

The positive root of the corresponding equality is:
\begin{equation*}
\epsilon^+ = \frac{\log(2/\delta_{MC})}{3} + \sqrt{\frac{\log^2(2/\delta_{MC})}{9} + \frac{T\log(2/\delta_{MC})}{2M}}.
\end{equation*}

Therefore, with probability at least $1 - \delta_{MC}$:
\begin{equation*}
|S_T| \leq \frac{\log(2/\delta_{MC})}{3} + \sqrt{\frac{\log^2(2/\delta_{MC})}{9} + \frac{T\log(2/\delta_{MC})}{2M}} \leq \frac{2\log(2/\delta_{MC})}{3} + \sqrt{\frac{T\log(2/\delta_{MC})}{2M}},
\end{equation*}
\noindent where the last (optional) inequality uses $\sqrt{a + b} \leq \sqrt{a} + \sqrt{b}$ to obtain a cleaner rate analysis.

Using Eq.~\ref{Eq: inequality_between_surrogate_to_target_integrated_var} and the algorithm's selection rule $\sigma_{t-1}^2(x_t) \geq \hat{\mathbb{E}}_{\tilde{p}_{t-1}}[\sigma_{t-1}^2]$,
\begin{align*}
\mathbb{E}_{P_{\mu_{t-1}}}[\sigma_{t-1}^2] &\leq \mathbb{E}_{\tilde{p}_{t-1}}[\sigma_{t-1}^2] \\
&= \hat{\mathbb{E}}_{\tilde{p}_{t-1}}[\sigma_{t-1}^2] - \left(\hat{\mathbb{E}}_{\tilde{p}_{t-1}}[\sigma_{t-1}^2] - \mathbb{E}_{\tilde{p}_{t-1}}[\sigma_{t-1}^2]\right) \\
&\leq \sigma_{t-1}^2(x_t) - \left(\hat{\mathbb{E}}_{\tilde{p}_{t-1}}[\sigma_{t-1}^2] - \mathbb{E}_{\tilde{p}_{t-1}}[\sigma_{t-1}^2]\right) \\& \leq C_1\gamma_T + \left|\sum_{t=1}^T \left(\hat{\mathbb{E}}_{\tilde{p}_{t-1}}[\sigma_{t-1}^2] - \mathbb{E}_{\tilde{p}_{t-1}}[\sigma_{t-1}^2]\right)\right|.
\end{align*}
Summing over $t = 1, \ldots, T$ and applying Lemma~\ref{lem:info_gain} and Part 2:
\begin{equation}
\sum_{t=1}^T \mathbb{E}_{P_{\mu_{t-1}}}[\sigma_{t-1}^2] \leq C_1\gamma_T + |S_T| \leq C_1\gamma_T + \frac{2\log(2/\delta_{MC})}{3} + \sqrt{\frac{T\log(2/\delta_{MC})}{2M}}
\label{eq:cumsum_bound}
\end{equation}

By monotonicity of posterior variance, $\sigma_T^2(x) \leq \sigma_{t-1}^2(x)$ for all $x \in \mathcal{X}$ and $t \leq T$:
\begin{equation*}
T \cdot \mathbb{E}_{P_{\mu_T}}[\sigma_T^2] \leq \sum_{t=1}^T \mathbb{E}_{P_{\mu_T}}[\sigma_{t-1}^2].
\end{equation*}
To relate $\mathbb{E}_{P_{\mu_T}}[\sigma_{t-1}^2]$ to $\mathbb{E}_{P_{\mu_{t-1}}}[\sigma_{t-1}^2]$, we bound the density ratio. By Lemma~\ref{lem:gp_confidence}, with probability at least $1 - \delta_{\mathrm{GP}}$, simultaneously for all $t \le T$ and $x \in \mathcal{X}$ (on the uniform confidence event defined above):
\begin{equation*}
|\mu_T(x) - \mu_{t-1}(x)| \leq |f(x) - \mu_T(x)| + |f(x) - \mu_{t-1}(x)| \leq 2\beta^{1/2}_T + 2\epsilon_T,
\end{equation*}
where we used $\sigma_T(x), \sigma_{t-1}(x) \leq 1$.

Applying Lemma~\ref{lem:log_density_ratio} with $g_1 = \mu_T$, $g_2 = \mu_{t-1}$, $\Delta(x) \leq 2\beta^{1/2}_T + 2\epsilon_T$, and $\Delta_{\max} \leq 2\beta^{1/2}_T + 2\epsilon_T$:
\begin{equation*}
\frac{P_{\mu_T}(x)}{P_{\mu_{t-1}}(x)} \leq e^{4|\lambda|(\beta^{1/2}_T + \epsilon_T)}.
\end{equation*}
Therefore:
\begin{equation*}
\mathbb{E}_{P_{\mu_T}}[\sigma_{t-1}^2] = \mathbb{E}_{P_{\mu_{t-1}}}\left[\sigma_{t-1}^2 \cdot \frac{P_{\mu_T}}{P_{\mu_{t-1}}}\right] \leq e^{4|\lambda|(\beta^{1/2}_T + \epsilon_T)} \cdot \mathbb{E}_{P_{\mu_{t-1}}}[\sigma_{t-1}^2].
\end{equation*}
Summing over $t$ and applying Eq.~\ref{eq:cumsum_bound}:
\begin{align*}
T \cdot \mathbb{E}_{P_{\mu_T}}[\sigma_T^2] &\leq e^{4|\lambda|(\beta^{1/2}_T + \epsilon_T)} \cdot \sum_{t=1}^T \mathbb{E}_{P_{\mu_{t-1}}}[\sigma_{t-1}^2] \\
&\leq e^{4|\lambda|(\beta^{1/2}_T + \epsilon_T)} \cdot \left(C_1\gamma_T + \frac{2\log(2/\delta_{MC})}{3} + \sqrt{\frac{T\log(2/\delta_{MC})}{2M}}\right).
\end{align*}
Dividing by $T$:
\begin{equation*}
\mathbb{E}_{P_{\mu_T}}[\sigma_T^2] \leq e^{4|\lambda|(\beta^{1/2}_T + \epsilon_T)} \cdot \left(\frac{C_1\gamma_T}{T} + \frac{2\log(2/\delta_{MC})}{3T} + \sqrt{\frac{\log(2/\delta_{MC})}{2MT}}\right).
\end{equation*}

For finite $\mathcal{X}$ with exact threshold computation, $\hat{\mathbb{E}}_{\tilde{p}_{t-1}}[\sigma_{t-1}^2] = \mathbb{E}_{\tilde{p}_{t-1}}[\sigma_{t-1}^2]$ for all $t$, so $\epsilon_{\mathrm{MC}} = \delta_{\mathrm{MC}} = 0$.

\end{proof}

\begin{lemma}[Variance Bound for TS-SID-iVAR]\label{lem:var_bound_ts}
Consider \texttt{TS-SID-iVAR} selecting $x_t \in \bar{\mathcal{X}}_t$ where $\bar{\mathcal{X}}_t$ is defined in Table~\ref{tab:methods}. If the constraint threshold at each iteration is estimated via $M$ Monte Carlo samples, then for any $\delta_{\mathrm{GP}}, \delta_{\mathrm{MC}} \in (0,1)$, with probability at least $1 - \delta_{\mathrm{GP}} - \delta_{\mathrm{MC}}$:
\begin{equation*}
\mathbb{E}_{P_{\mu_T}}[\sigma_T^2] \leq e^{6|\lambda|(\beta^{1/2}_T + \epsilon_T)} \cdot \left(\frac{C_1\gamma_T}{T} + \frac{2\log(2/\delta_{\mathrm{MC}})}{3T} + \sqrt{\frac{\log(2/\delta_{\mathrm{MC}})}{2MT}}\right),
\end{equation*}
where $\beta_T, \epsilon_T$ are as in Lemma~\ref{lem:gp_confidence} with failure probability $\delta_{\mathrm{GP}}$. For finite $\mathcal{X}$ with exact threshold computation, the $\log(2/\delta_{\mathrm{MC}})$ terms vanish.
\end{lemma}

\begin{proof}
The proof follows a similar structure to Lemma~\ref{lem:var_bound_ab}, with an additional step to handle the Thompson sampling randomness.

At each round $t$, \texttt{TS-SID}  draws a posterior sample $\tilde{f}_{t-1} \sim f|D_{t-1}$ and defines $\tilde{p}_{t-1} \propto e^{\lambda\tilde{f}_{t-1} + b}$.
By Lemma~\ref{lem:gp_confidence} applied to the posterior samples $\{\tilde{f}_{t-1}\}_{t\le T}$ (which, conditionally on $D_{t-1}$, share the law of $f \mid D_{t-1}$) on the analogous uniform confidence event, with $\delta_{\mathrm{GP}}$ split equally between this event and the one for $f$ used in the distribution-shift step below, with probability at least $1 - \delta_{\mathrm{GP}}$:
\begin{equation*}
|\tilde{f}_{t-1}(x) - \mu_{t-1}(x)| \leq \beta^{1/2}_T \sigma_{t-1}(x) + \epsilon_T \leq \beta^{1/2}_T + \epsilon_T
\end{equation*}
for all $x \in \mathcal{X}$ and $t \leq T$. Applying Lemma~\ref{lem:log_density_ratio}:
\begin{equation}\label{eq:ts_density_ratio}
\frac{P_{\mu_{t-1}}(x)}{P_{\tilde{f}_{t-1}}(x)} \leq e^{2|\lambda|(\beta^{1/2}_T + \epsilon_T)}.
\end{equation}
This implies:
\begin{equation*}
\mathbb{E}_{P_{\tilde{f}_{t-1}}}[\sigma_{t-1}^2] \geq e^{-2|\lambda|(\beta^{1/2}_T + \epsilon_T)} \cdot \mathbb{E}_{P_{\mu_{t-1}}}[\sigma_{t-1}^2].
\end{equation*}

The algorithm estimates $\mathbb{E}_{P_{\tilde{f}_{t-1}}}[\sigma_{t-1}^2]$ via MC samples $x^{(i)} \stackrel{\text{i.i.d.}}{\sim} P_{\tilde{f}_{t-1}}$. Conditioned on $\tilde{\mathcal{F}}_{t-1}$ (which includes both $D_{t-1}$ and $\tilde{f}_{t-1}$), the distribution $P_{\tilde{f}_{t-1}}$ is fixed, so the estimation errors

\begin{equation*}
\xi_t := \hat{\mathbb{E}}_{P_{\tilde{f}_{t-1}}}[\sigma_{t-1}^2] - \mathbb{E}_{P_{\tilde{f}_{t-1}}}[\sigma_{t-1}^2]
\end{equation*}

form a martingale difference sequence with respect to $\{\tilde{\mathcal{F}}_t\}$. Following the same argument as in  Lemma~\ref{lem:var_bound_ab} Part 2, with probability at least $1 - \delta_{MC}$:

\begin{equation*}
|S_T| \leq \frac{2\log(2/\delta_{MC})}{3} + \sqrt{\frac{T\log(2/\delta_{MC})}{2M}}.
\end{equation*}

The same argument as Lemma~\ref{lem:var_bound_ab} Part 3 applies with $P_{\tilde{f}_{t-1}}$ in place of $\tilde{p}_{t-1}$, yielding:
\begin{equation*}
\sum_{t=1}^T \mathbb{E}_{P_{\tilde{f}_{t-1}}}[\sigma_{t-1}^2] \leq C_1\gamma_T + |S_T|.
\end{equation*}
Unlike \texttt{AB-SID}, here we need convert from $P_{\tilde{f}_{t-1}}$ to $P_{\mu_{t-1}}$. Using the density ratio defined in Eq.~\ref{eq:ts_density_ratio}:
\begin{equation*}
\mathbb{E}_{P_{\mu_{t-1}}}[\sigma_{t-1}^2] = \mathbb{E}_{P_{\tilde{f}_{t-1}}}\left[\sigma_{t-1}^2 \cdot \frac{P_{\mu_{t-1}}}{P_{\tilde{f}_{t-1}}}\right] \leq e^{2|\lambda|(\beta^{1/2}_T + \epsilon_T)} \cdot \mathbb{E}_{P_{\tilde{f}_{t-1}}}[\sigma_{t-1}^2].
\end{equation*}
Summing over $t$:
\begin{equation}\label{eq:ts_cumsum_bound}
\sum_{t=1}^T \mathbb{E}_{P_{\mu_{t-1}}}[\sigma_{t-1}^2] \leq e^{2|\lambda|(\beta^{1/2}_T + \epsilon_T)} \cdot \left(C_1\gamma_T + |S_T|\right).
\end{equation}
This additional factor of $e^{2|\lambda|(\beta^{1/2}_T + \epsilon_T)}$ arises from the Thompson sampling approximation and is absent in \texttt{AB-SID}.

The same argument as Lemma~\ref{lem:var_bound_ab} Part 4 gives:
\begin{equation*}
\mathbb{E}_{P_{\mu_T}}[\sigma_{t-1}^2] \leq e^{4|\lambda|(\beta^{1/2}_T + \epsilon_T)} \cdot \mathbb{E}_{P_{\mu_{t-1}}}[\sigma_{t-1}^2].
\end{equation*}
Summing over $t$ and applying \eqref{eq:ts_cumsum_bound}:
\begin{align*}
T \cdot \mathbb{E}_{P_{\mu_T}}[\sigma_T^2] &\leq e^{4|\lambda|(\beta^{1/2}_T + \epsilon_T)} \cdot \sum_{t=1}^T \mathbb{E}_{P_{\mu_{t-1}}}[\sigma_{t-1}^2] \\
&\leq e^{4|\lambda|(\beta^{1/2}_T + \epsilon_T)} \cdot e^{2|\lambda|(\beta^{1/2}_T + \epsilon_T)} \cdot \left(C_1\gamma_T + |S_T|\right) \\
&= e^{6|\lambda|(\beta^{1/2}_T + \epsilon_T)} \cdot \left(C_1\gamma_T + |S_T|\right).
\end{align*}
Dividing by $T$ and substituting the bound from Part 2:
\begin{equation*}
\mathbb{E}_{P_{\mu_T}}[\sigma_T^2] \leq e^{6|\lambda|(\beta^{1/2}_T + \epsilon_T)} \cdot \left(\frac{C_1\gamma_T}{T} + \frac{2\log(2/\delta_{MC})}{3T} + \sqrt{\frac{\log(2/\delta_{MC})}{2MT}}\right).
\end{equation*}

For finite $\mathcal{X}$ with exact threshold computation, $\hat{\mathbb{E}}_{\tilde{p}_{t-1}}[\sigma_{t-1}^2] = \mathbb{E}_{P_{\tilde{f}_{t-1}}}[\sigma_{t-1}^2]$ for all $t$, so $\epsilon_{\mathrm{MC}} = \delta_{\mathrm{MC}} = 0$.
\end{proof}

We are now able to prove Theorem~\ref{thm:hp_terminal}.
\begin{proof}
Set $\delta_{\mathrm{GP}} = \delta_{\mathrm{MC}} = \delta/2$. By union bound, the following holds with probability at least $1 - \delta$.

By Lemma~\ref{lem:hp_mse_bound} with failure probability $\delta/2$:
\begin{equation*}
\mathbb{E}_{P_f}[(f - \mu_T)^2] \leq e^{2|\lambda|(\beta^{1/2}_T\sigma_{\max,T} + \epsilon_T)} \cdot \left(\beta_T \cdot \mathbb{E}_{P_{\mu_T}}[\sigma_T^2] + 2\beta^{1/2}_T\epsilon_T + \epsilon_T^2\right).
\end{equation*}
Using $\sigma_{\max,T} \leq 1$ (kernel normalization):
\begin{equation}\label{eq:mse_to_var}
\mathbb{E}_{P_f}[(f - \mu_T)^2] \leq e^{2|\lambda|(\beta^{1/2}_T + \epsilon_T)} \cdot \left(\beta_T \cdot \mathbb{E}_{P_{\mu_T}}[\sigma_T^2] + 2\beta^{1/2}_T\epsilon_T + \epsilon_T^2\right).
\end{equation}

By Lemma~\ref{lem:var_bound_ab} (for \texttt{AB-SID}) or Lemma~\ref{lem:var_bound_ts} (for TS-SID) with failure probability $\delta/2$:
\begin{equation*}
\mathbb{E}_{P_{\mu_T}}[\sigma_T^2] \leq e^{2(1+\kappa)|\lambda|(\beta^{1/2}_T + \epsilon_T)} \cdot \left(\frac{C_1\gamma_T}{T} + \frac{2\log(4/\delta)}{3T} + \sqrt{\frac{\log(4/\delta)}{2MT}}\right),
\end{equation*}
where $\kappa = 1$ for \texttt{AB-SID}(prefactor $e^{4|\lambda|(\cdot)}$) and $\kappa = 2$ for \texttt{TS-SID}  (prefactor $e^{6|\lambda|(\cdot)}$). Denoting the \textit{cumulative MC error} $\epsilon_{\mathrm{MC}} := \frac{2\log(4/\delta)}{3T} + \sqrt{\frac{\log(4/\delta)}{2MT}}$ and substituting into \eqref{eq:mse_to_var}:
\begin{align*}
\mathbb{E}_{P_f}[(f - \mu_T)^2] &\leq e^{2|\lambda|(\beta^{1/2}_T + \epsilon_T)} \cdot \left(\beta_T \cdot e^{2(1+\kappa)|\lambda|(\beta^{1/2}_T + \epsilon_T)} \cdot \left(\frac{C_1\gamma_T}{T} + \epsilon_{\mathrm{MC}}\right) + 2\beta^{1/2}_T\epsilon_T + \epsilon_T^2\right).
\end{align*}

Using $e^{2|\lambda|(\cdot)} \cdot e^{2(1+\kappa)|\lambda|(\cdot)} = e^{2(\kappa+2)|\lambda|(\cdot)}$:
\begin{align*}
\mathbb{E}_{P_f}[(f - \mu_T)^2] &\leq e^{2(\kappa+2)|\lambda|(\beta^{1/2}_T + \epsilon_T)} \cdot \left(\frac{C_1\beta_T\gamma_T}{T} + \beta_T\epsilon_{\mathrm{MC}}\right) + e^{2|\lambda|(\beta^{1/2}_T + \epsilon_T)} \cdot \left(2\beta^{1/2}_T\epsilon_T + \epsilon_T^2\right).
\end{align*}

Since $e^{2|\lambda|(\cdot)} \leq e^{2(\kappa+2)|\lambda|(\cdot)}$ for $\kappa \geq 0$:
\begin{equation*}
\mathbb{E}_{P_f}[(f - \mu_T)^2] \leq e^{2(\kappa+2)|\lambda|(\beta^{1/2}_T + \epsilon_T)} \cdot \left(\frac{C_1\beta_T\gamma_T}{T} + \beta_T\epsilon_{\mathrm{MC}} + 2\beta^{1/2}_T\epsilon_T + \epsilon_T^2\right).
\end{equation*}
\end{proof}

This enables the following convergence rate analysis:

\begin{proof}
Since $\beta_T = \mathcal{O}(\log T)$, we have $\sqrt{\beta_T} = \mathcal{O}(\sqrt{\log T})$. For any $\epsilon > 0$: $e^{c\sqrt{\log T}} = o(T^\epsilon)$
since $\sqrt{\log T} = o(\log T)$. Thus the prefactor satisfies $e^{\mathcal{O}(\sqrt{\beta_T})} = o(T^\epsilon)$.

For the MC error:
\begin{itemize}
    \item If $M$ is constant: $\beta_T \epsilon_{\mathrm{MC}} = \mathcal{O}\left(\frac{\log T}{\sqrt{T}}\right) = o(T^{-1/2 + \epsilon})$.
    \item If $M = \Omega(T)$: $\beta_T \epsilon_{\mathrm{MC}} = \mathcal{O}\left(\frac{\log T}{T}\right) = o(T^{-1 + \epsilon})$.
\end{itemize}

Combining the prefactor with the dominant term, for any $\epsilon > 0$:
\begin{itemize}
    \item If $M$ is constant: $o(T^\epsilon) \cdot o(T^{-1/2 + \epsilon}) = \mathcal{O}(T^{-1/2 + 2\epsilon})$.
    \item If $M = \Omega(T)$: $o(T^\epsilon) \cdot o(T^{-1 + \epsilon}) = \mathcal{O}(T^{-1 + 2\epsilon})$.
\end{itemize}
Since $\epsilon > 0$ is arbitrary, the result follows.
\end{proof}

\subsection{Theorem~\ref{thm:avg_mse}}
\label{App: Bayes_regret}
\begin{proof}[Proof of Theorem~\ref{thm:avg_mse}]
By Lemma~\ref{lem:bayesian_mse}, conditional on
the terminal dataset $\mathcal{D}_T$, we have
\begin{equation*}
\mathbb{E}_{f\mid\mathcal{D}_T}
\left[\mathbb{E}_{P_f}[(f-\mu_T)^2]\right]
\leq
e^{2\lambda^2}(1+4\lambda^2)
\cdot \mathbb{E}_{P_{\mu_T}}[\sigma_T^2].
\end{equation*}
By Lemma~\ref{lem:var_bound_ab} (for \texttt{AB-SID})
or Lemma~\ref{lem:var_bound_ts} (for \texttt{TS-SID}),
with probability at least $1-\delta$:
\begin{equation*}
\mathbb{E}_{P_{\mu_T}}[\sigma_T^2]
\leq
e^{2(1+\kappa)|\lambda|(\beta_T^{1/2}+\epsilon_T)}
\cdot
\left(
\frac{C_1\gamma_T}{T}+\epsilon_{\mathrm{MC}}
\right),
\end{equation*}
where $\kappa=1$ for \texttt{AB-SID} and
$\kappa=2$ for \texttt{TS-SID}.
Substituting completes the proof.
\end{proof}

\begin{proof}[Proof of Corollary~\ref{cor:bayesian_terminal}]
For notation simplicity, let
\[
R_T(\mathcal{D}_T)
:=
\mathbb{E}_{f\mid\mathcal{D}_T}
[\mathcal{L}(\mu_T)],
\qquad
C_\lambda:=e^{2\lambda^2}(1+4\lambda^2),
\]
and let $B_T(\delta)$ denote the right-hand side of
Theorem~\ref{thm:avg_mse}.
By Lemma~\ref{lem:bayesian_mse} and the kernel normalization
$k(x,x)\leq 1$, we have
\[
0\leq R_T(\mathcal{D}_T)
\leq C_\lambda\mathbb{E}_{P_{\mu_T}}[\sigma_T^2]
\leq C_\lambda.
\]
For any fixed $\delta\in(0,1)$, define
\[
\mathcal{E}_\delta
:=
\left\{R_T(\mathcal{D}_T)\leq B_T(\delta)\right\}.
\]
By Theorem~\ref{thm:avg_mse},
$\mathbb{P}(\mathcal{E}_\delta)\geq 1-\delta$.
Splitting the expectation over this event and its complement,
\begin{align*}
\mathbb{E}_{\mathcal{D}_T}[R_T(\mathcal{D}_T)]
&=
\mathbb{E}[R_T(\mathcal{D}_T)\mathbf{1}_{\mathcal{E}_\delta}]
+
\mathbb{E}[R_T(\mathcal{D}_T)\mathbf{1}_{\mathcal{E}_\delta^c}]\\
&\leq
B_T(\delta)\mathbb{P}(\mathcal{E}_\delta)
+
C_\lambda\mathbb{P}(\mathcal{E}_\delta^c)\\
&\leq B_T(\delta)+C_\lambda\delta.
\end{align*}
Finally, by the law of total expectation under the joint
Bayesian distribution of $f$ and the adaptive sampling history,
\[
\mathbb{E}_{\mathcal{D}_T}[R_T(\mathcal{D}_T)]
=
\mathbb{E}_f\left[
\mathbb{E}_{\mathcal{D}_T\mid f}
[\mathcal{L}(\mu_T)]
\right].
\]
Here, $\mathcal{D}_T\mid f$ denotes the sampling history
generated by the algorithm for a fixed $f$, including
observation noise and algorithmic randomness.
This completes the proof.
\end{proof}

\subsection{Rate Analysis for Posterior-Averaged and Bayesian MSE Bounds}
\label{app:avg_rate}
\paragraph{Posterior-averaged bound.}
The posterior-averaged bound (Theorem~\ref{thm:avg_mse})
achieves the same asymptotic rates as the high-probability bound
(Corollary~\ref{cor:rate}), under polylogarithmic information-gain
growth $\gamma_T=\mathcal{O}((\log T)^q)$ for some fixed $q\geq 0$.\\
The prefactor $e^{2\lambda^2}(1+4\lambda^2)$ is constant in $T$.
For the remaining prefactor, since $\beta_T=\mathcal{O}(\log T)$,
we have
$e^{2(1+\kappa)|\lambda|(\beta_T^{1/2}+\epsilon_T)}
=e^{\mathcal{O}(\sqrt{\log T})}=o(T^\epsilon)$
for any $\epsilon>0$.\\
For the MC error:
\begin{itemize}
    \item If $M$ is constant:
    $\epsilon_{\mathrm{MC}}=\mathcal{O}(1/\sqrt{T})$.
    \item If $M=\Omega(T)$:
    $\epsilon_{\mathrm{MC}}=\mathcal{O}(1/T)$.
\end{itemize}
Combining the prefactor with the information-gain and MC terms
yields $\mathcal{O}(T^{-1/2+\epsilon})$ for constant $M$
and $\mathcal{O}(T^{-1+\epsilon})$ for $M=\Omega(T)$.\\
Compared to the high-probability realized-MSE bound, the
posterior-averaged bound has tighter dependence on the
confidence parameter: the main term scales as $\gamma_T/T$
instead of $\beta_T\gamma_T/T$, and the prefactor exponent
is reduced from $2(\kappa+2)|\lambda|$ to $2(\kappa+1)|\lambda|$.
\paragraph{Bayesian terminal MSE bound.}
For Corollary~\ref{cor:bayesian_terminal}, choose the auxiliary
confidence parameter $\delta_T=(T+1)^{-2}$. The additive term
$C_\lambda\delta_T$ is then $\mathcal{O}(T^{-2})$.
Evaluating the confidence-dependent quantities at $\delta_T$ gives
\[
\beta_T=\mathcal{O}(\log T),\qquad
\epsilon_T=\mathcal{O}\left(\frac{\sqrt{\log T}}{T}\right),
\]
and, for continuous domains,
\[
\epsilon_{\mathrm{MC}}=\mathcal{O}\left(
\frac{\log T}{T}+\sqrt{\frac{\log T}{MT}}
\right).
\]
The exponential prefactor remains
$e^{\mathcal{O}(\sqrt{\log T})}=o(T^\epsilon)$ for every
$\epsilon>0$. The fully averaged Bayesian terminal MSE is therefore
bounded by
\[
\mathcal{O}\left[
e^{\mathcal{O}(\sqrt{\log T})}
\left(\frac{\gamma_T}{T}+\frac{\log T}{T}
+\sqrt{\frac{\log T}{MT}}\right)+T^{-2}\right].
\]
Under the same polylogarithmic information-gain condition,
this recovers the rates $\mathcal{O}(T^{-1/2+\epsilon})$
for constant $M$ and $\mathcal{O}(T^{-1+\epsilon})$
for $M=\Omega(T)$, now for the unconditional Bayesian expected MSE.
For finite domains with exact threshold computation,
$\epsilon_T=\epsilon_{\mathrm{MC}}=0$.\\
Compared with the high-probability realized-MSE bound in
Theorem~\ref{thm:hp_terminal}, the posterior-averaged bound
has $\gamma_T/T$ rather than $\beta_T\gamma_T/T$ as its
main term and reduces the coefficient in the exponential
prefactor from $2(\kappa+2)|\lambda|$ to $2(\kappa+1)|\lambda|$.

\subsection{Proposition~\ref{prop:fixed-threshold}}

Analogous to our previous theoretical analysis, we first establish an upper bound for fixed-threshold methods, then derive the convergence rate. The key observation is that any algorithm querying points with $\sigma^2_{t-1}(x_t) \geq \eta$ must exit this regime within $\mathcal{O}(\gamma_T/\eta)$ iterations by the information gain bound, after which the posterior variance remains uniformly below $\eta$ by monotonicity.

\begin{lemma}[High-Probability MSE Upper Bound for Fixed-Threshold Methods]
\label{lem:fixed-threshold-bound}
Consider any active learning algorithm that selects $x_t$ satisfying
$\sigma^2_{t-1}(x_t) \geq \eta$ when $\bar{\mathcal{X}}'_t := \{x : \sigma^2_{t-1}(x) \geq \eta\} \neq \phi$, and selects $x_t = \arg\max_{x \in \mathcal{X}} \sigma_{t-1}(x)$ otherwise, where $\eta \in (0,1]$ is a fixed threshold. Under Assumptions~\ref{ass:kernel_regularity} and \ref{ass:target_dist}, let $\mathcal{X} \subset [0,r]^d$ be compact. Then for any $\delta \in (0,1)$, with probability at least $1 - \delta$:
\begin{equation*}
L(\mu_T) \leq e^{2|\lambda|(\beta_T^{1/2}\sqrt{\eta} + \epsilon_T)} \cdot
\left(\beta_T \eta + 2\beta_T^{1/2}\epsilon_T + \epsilon_T^2\right)
\end{equation*}
provided $\eta > C_1\gamma_T / T$, where $\beta_T, \epsilon_T$ are as in
Lemma~\ref{lem:gp_confidence} with failure probability $\delta$, and
$C_1 = 2/\log(1 + \tau^{-2})$. For finite $\mathcal{X}$, $\epsilon_T = 0$.
\end{lemma}

\begin{proof}
 According to the algorithm, we define the AL iteration as the following two phases:
\begin{equation*}
T^{(1)} := \{t \leq T : \max_{x \in \mathcal{X}} \sigma^2_{t-1}(x) \geq \eta\}, \quad T^{(2)} := \{t \leq T : \max_{x \in \mathcal{X}} \sigma^2_{t-1}(x) < \eta\}.
\end{equation*}

 In Phase 1, by the algorithm's selection rule, any query point satisfies $\sigma^2_{t-1}(x_t) \geq \eta$. By Lemma~\ref{lem:info_gain}:
\begin{equation*}
|T^{(1)}| \cdot \eta \leq \sum_{t \in T^{(1)}} \sigma^2_{t-1}(x_t) \leq \sum_{t=1}^{T} \sigma^2_{t-1}(x_t) \leq C_1\gamma_T.
\end{equation*}
Thus $|T^{(1)}| \leq C_1\gamma_T / \eta < T$ under the stated condition on $\eta$, ensuring Phase 2 is non-empty.

 Once Phase 2 begins, $\max_{x \in \mathcal{X}} \sigma^2_{t-1}(x) < \eta$ by definition. Since posterior variance is monotonically non-increasing with additional observations, $\sigma_{\max,T} < \sqrt{\eta}$ and $\mathbb{E}_{P_{\mu_T}}[\sigma^2_T(x)] < \eta$.

 By Lemma~\ref{lem:gp_confidence}, with probability at least $1 - \delta$, the confidence bound holds uniformly. Applying Lemma~\ref{lem:hp_mse_bound}:
\begin{equation*}
L(\mu_T) \leq e^{2|\lambda|(\beta_T^{1/2}\sigma_{\max,T} + \epsilon_T)} \cdot \left(\beta_T \cdot \mathbb{E}_{P_{\mu_T}}[\sigma^2_T(x)] + 2\beta_T^{1/2}\epsilon_T + \epsilon_T^2\right).
\end{equation*}
Substituting $\sigma_{\max,T} < \sqrt{\eta}$ and $\mathbb{E}_{P_{\mu_T}}[\sigma^2_T] < \eta$ yields the result.
\end{proof}

Now we start to prove proposition~\ref{prop:fixed-threshold}.

\begin{proof}
By Lemma~\ref{lem:fixed-threshold-bound}, Phase 1 contains at most $|T^{(1)}| \leq C_1\gamma_T/\eta$ iterations. In Phase 2, the algorithm selects $x_t = \arg\max_{x \in \mathcal{X}} \sigma_{t-1}(x)$, so $\sigma^2_{\max,t}$ is non-increasing for $t \in T^{(2)}$. Applying the information gain bound to Phase 2:
\begin{equation*}
|T^{(2)}| \cdot \sigma^2_{\max,T} \leq \sum_{t \in T^{(2)}} \sigma^2_{t-1}(x_t) \leq C_1\gamma_T.
\end{equation*}
Since $|T^{(2)}| = T - |T^{(1)}| \geq T - C_1\gamma_T/\eta$, we obtain:
\begin{equation*}
\sigma^2_{\max,T} \leq \frac{C_1\gamma_T}{T - C_1\gamma_T/\eta}.
\end{equation*}
For any fixed $\eta \in (0,1]$, since $\gamma_T = o(T)$ for standard kernels, we have $T - C_1\gamma_T/\eta \geq T/2$ for sufficiently large $T$, yielding $\sigma^2_{\max,T} \leq 2C_1\gamma_T/T$.

Applying Lemma~\ref{lem:hp_mse_bound} with $\sigma_{\max,T} \leq \sqrt{2C_1\gamma_T/T}$ and $\mathbb{E}_{P_{\mu_T}}[\sigma^2_T] \leq \sigma^2_{\max,T} \leq 2C_1\gamma_T/T$:
\begin{equation*}
L(\mu_T) \leq e^{2|\lambda|(\beta_T^{1/2}\sigma_{\max,T} + \epsilon_T)} \cdot \left(\beta_T \cdot \frac{2C_1\gamma_T}{T} + 2\beta_T^{1/2}\epsilon_T + \epsilon_T^2\right).
\end{equation*}
For the prefactor, since $\beta_T = \mathcal{O}(\log T)$, $\gamma_T = \mathcal{O}((\log T)^c)$ for kernel-dependent constant $c$, and $\epsilon_T = \mathcal{O}(1/T)$:
\begin{equation*}
\beta_T^{1/2}\sigma_{\max,T} = \mathcal{O}\left(\sqrt{\log T} \cdot \sqrt{\frac{(\log T)^c}{T}}\right) = \mathcal{O}\left(\frac{(\log T)^{(c+1)/2}}{\sqrt{T}}\right) = o(1).
\end{equation*}
Thus $e^{2|\lambda|(\beta_T^{1/2}\sigma_{\max,T} + \epsilon_T)} = 1 + o(1)$. The dominant term is $\beta_T \cdot 2C_1\gamma_T/T = \mathcal{O}((\log T)^{c+1}/T)$, while the remaining terms satisfy $2\beta_T^{1/2}\epsilon_T + \epsilon_T^2 = \mathcal{O}(\sqrt{\log T}/T) = o((\log T)^{c+1}/T)$. Therefore $L(\mu_T) = \mathcal{O}(T^{-1+\epsilon})$ for arbitrarily small $\epsilon > 0$.
\end{proof}

\subsection{Robustness to Boltzmann Parameter Misspecification} \label{App: Boltzmann_misspecify}
\begin{proposition}[Robustness to Boltzmann parameter misspecification]
\label{prop:misspecification}
Let
\[
P_f^{(i)}(x)
=
\frac{1}{Z_f^{(i)}}
\exp\ \bigl(\lambda_i f(x)+b_i(x)\bigr),
\qquad i\in\{1,2\},
\]
denote the self-induced distributions associated with parameters
$(\lambda_i,b_i)$. For parameters  $(\lambda_1,b_1)$, $(\lambda_2,b_2)$ that has bounded difference
\[
|\lambda_2-\lambda_1|\leq \epsilon_\lambda,
\qquad
\|b_2-b_1\|_\infty\leq \epsilon_b,
\]
and additionally with bounded function range over $\mathcal{X}$:
\[
R_f
:=
\sup_{x\in\mathcal X} f(x)
-
\inf_{x\in\mathcal X} f(x)
<\infty.
\]
Then, for any surrogate $\hat f$,
\[
\mathbb E_{x\sim P_f^{(2)}}
\left[
  \bigl(\hat f(x)-f(x)\bigr)^2
\right]
\leq
\exp\!\left(
  \epsilon_\lambda R_f + 2\epsilon_b
\right)
\mathbb E_{x\sim P_f^{(1)}}
\left[
  \bigl(\hat f(x)-f(x)\bigr)^2
\right].
\]
Consequently, under the conditions of Theorem~\ref{thm:hp_terminal},
its high-probability guarantee and the convergence rates in
Corollary~\ref{cor:rate} continue to hold under bounded parameter
misspecification, up to the multiplicative constant above.

\end{proposition}

\begin{proof}
Define
\[
h(x)
:=
(\lambda_2-\lambda_1)f(x)
+
b_2(x)-b_1(x).
\]
The density ratio between the test-time and acquisition-time
self-induced distributions can be written as
\[
\frac{P_f^{(2)}(x)}{P_f^{(1)}(x)}
=
\frac{Z_f^{(1)}}{Z_f^{(2)}}e^{h(x)}.
\]
Moreover,
\[
Z_f^{(2)}
=
\int_{\mathcal X}
e^{\lambda_1 f(x)+b_1(x)}e^{h(x)}\,dx
\geq
e^{\inf_{x\in\mathcal X}h(x)}
Z_f^{(1)}.
\]
Hence,
\[
\frac{P_f^{(2)}(x)}{P_f^{(1)}(x)}
\leq
\exp\!\left(
h(x)-\inf_{x'\in\mathcal X}h(x')
\right)
\leq
\exp\!\left(
\sup_{x\in\mathcal X}h(x)
-
\inf_{x\in\mathcal X}h(x)
\right).
\]
By the definition of $h$,
\begin{align*}
\sup_{x\in\mathcal X}h(x)
-
\inf_{x\in\mathcal X}h(x)
&\leq
|\lambda_2-\lambda_1|
\left(
\sup_{x\in\mathcal X}f(x)
-
\inf_{x\in\mathcal X}f(x)
\right)
\\
&\quad+
\sup_{x\in\mathcal X}\bigl(b_2(x)-b_1(x)\bigr)
-
\inf_{x\in\mathcal X}\bigl(b_2(x)-b_1(x)\bigr)
\\
&\leq
\epsilon_\lambda R_f + 2\epsilon_b.
\end{align*}
Therefore,
\[
\frac{P_f^{(2)}(x)}{P_f^{(1)}(x)}
\leq
\exp\!\left(
\epsilon_\lambda R_f+2\epsilon_b
\right)
\qquad
\text{for all }x\in\mathcal X.
\]
Changing measure from $P_f^{(2)}$ to $P_f^{(1)}$ now gives
\begin{align*}
\mathbb E_{x\sim P_f^{(2)}}
\left[
  \bigl(\hat f(x)-f(x)\bigr)^2
\right]
&=
\mathbb E_{x\sim P_f^{(1)}}
\left[
  \frac{P_f^{(2)}(x)}{P_f^{(1)}(x)}
  \bigl(\hat f(x)-f(x)\bigr)^2
\right]
\\
&\leq
\exp\!\left(
  \epsilon_\lambda R_f+2\epsilon_b
\right)
\mathbb E_{x\sim P_f^{(1)}}
\left[
  \bigl(\hat f(x)-f(x)\bigr)^2
\right].
\end{align*}
The multiplicative factor is independent of $T$ for fixed
$\epsilon_\lambda$, $\epsilon_b$, and $R_f$. Therefore, applying
Theorem~\ref{thm:hp_terminal} to the final expectation preserves the
convergence rates in Corollary~\ref{cor:rate}.
\end{proof}

\section{Experimental Details and Additional Experimental Investigations} \label{app: experimental_investigation}
\subsection{Experimental Setup Details}
\label{App: setup}

For all methods except the molecular drug discovery example, we model the objective
function using a GP with a Matérn-5/2 kernel, estimate hyperparameters via maximum
a posteriori (MAP) with the default priors from~\citet{hvarfner2024vanilla}, and fix
observation noise at $\tau^2 = 10^{-4}$, except for the GP prior problem where we
assume the true kernel hyperparameters are known. We pre-normalize inputs to $[0,1]^d$
and standardize outputs online. For acquisition function optimization over continuous
domains, we use SLSQP~\citep{kraft1988software} to enforce the variance constraint
defining $\bar{\mathcal{X}}_t$ \footnote{In practice, constrained acquisition optimization was numerically stable with feasible initializations. On the Branin benchmark, only 9 out of 3,000 selections by AB-SID-iVAR across 30 seeds exhibited constraint violations, all within the numerical tolerance of $10^{-8}$, with no optimizer failures observed.}, and L-BFGS-B~\citep{byrd1995limited} otherwise, both
via \texttt{BoTorch}~\citep{balandat2020botorch}.

For continuous $\mathcal{X}$, we approximate the acquisition function integral via
Sequential Monte Carlo (SMC), which has demonstrated strong performance for sampling
Boltzmann-form distributions~\citep{akhound2024iterated, akhound2025progressive}.
Specifically, we use $N = 1000d$ particles, 10 temperature levels, and 10 random walk
Metropolis steps per level. For fair comparison, the same number of reference points
is used for all acquisition functions requiring integration over the input space.
A sensitivity study to particle number in terms of performance is also reported in Appendix~\ref{app: sensitivity_analysis_of_mcmc}.

\subsection{Generic Adaptation of Heuristics Active Learning Algorithms} \label{app:hal}
Several heuristic active learning approaches have been proposed in the computational materials science and potential energy modeling literature for efficient assembly of data for the training of machine-learned interatomic potentials (MLIP); See, e.g., \citep{li2015molecular,podryabinkin2017active,vandermause2020fly,jinnouchi2020fly,van2023hyperactive,kulichenko2023uncertainty,duschatko2024uncertainty}. These MLIPs are required to be accurate for configurations that are probable under the specific statistical ensembles targeted in the molecular dynamics simulation in which these MLIPs are intended to be used, but accuracy of the MLIP on configurations that are very improbable within this ensemble is of less importance. As such, the problem of fitting an MLIP resembles the here presented SIDAL problem and heuristic active learning algorithms designed for these problems are natural contenders to be considered in benchmark problems. Using the notation introduced in the main body of the paper, the generic form of a query of these algorithms consists of the following sampling procedure:
\begin{enumerate}
\item Run a single-chain Markov Chain Monte Carlo algorithm (e.g., a stochastic thermostat or constant-energy molecular dynamics) targeting a surrogate density $\tilde{p}_{t-1}$ of the generic form
\[
\tilde{p}_{t-1}(x) \propto \exp\bigl(\lambda\mu_{t-1}(x) + \lambda h(\sigma^2_{t-1}(x)) + b(x)\bigr),
\]
where $h(\sigma^2)$ is a smooth transformation of $\sigma^2$.
\item Once the Markov chain reaches a configuration where a method-dependent uncertainty measure $\UC(x)$ exceeds a preset threshold, $\eta>0$, the Markov chain terminates and the current sample is queried.
\end{enumerate}
Algorithms proposed in the literature differ in how these two steps are implemented, specifically in terms of the choice of the surrogate density $\tilde{p}_{t-1}(x)$, and the uncertainty measure $\UC$ (see Table \ref{tab:al_methods}), are often motivated by problem-specific properties and constraints and thus are not necessarily suited for the generic SIDAL setting considered here.

We therefore adapt the existing heuristics in the following Generic Heuristic Active Learning (GHAL) approach to fit the here considered generic SIDAL setting:

  \begin{enumerate}
      \item Run single-chain MCMC targeting $\tilde{p}_{t-1}(x)\propto \exp(\lambda(\mu_{t-1}(x) - \tau\sigma_{t-1}(x)))$ initialized at a point in input space close to the current training set. (This is to ensure that the uncertainty threshold is not exceeded at initialization.)
      \item If $\sigma^2_{t-1}(x) > \eta$, select this sample $x$. If the threshold is not exceeded after the maximum number of steps of the Markov chain, select the sample with highest predictive variance $\sigma^2_{t-1}(x)$.
      \item Query $y_t = f(x_t) + \epsilon_t$ and update the GP posterior
  \end{enumerate}
Note that for $\tau=0$, this methods closely resembles the query strategy in the active learning loop of FLARE \cite{vandermause2020fly}, whereas for $\tau>0$, this approach resembles the Hyperactive Learning HAL method \cite{van2023hyperactive} up to the form of the uncertainty measure $\UC$ used in the thresholding step.
Here, both $\tau>0$ and the threshold $\eta>0$ are hyper-parameters. To the best of our knowledge, there are no systematic apriori tuning strategies described in the literature. Therefore, we resort to randomly sample the values of these hyperparameters uniformly from a range of plausible values at the beginning of each active learning loop, i.e.,  $\tau \sim \mathcal{U}( [0, 3])$ and $\eta \sim  \mathcal{U}( [1, 5])$.

  We mention that both HAL and the here proposed methods modify the sampling distribution, but with fundamentally different mechanisms and goals. HAL uses a \emph{single-chain reactive} approach: sampling from a heuristically biased distribution and triggering queries when local uncertainty exceeds a threshold. In contrast, our method uses \emph{multi-particle proactive} optimization: maintaining $N$ particles via SMC on the surrogate $\tilde{p}_{t-1}$ of the target distribution $P_f \propto \exp(\lambda f)$, then selecting points that maximize variance reduction under this distribution. The linear $+\tau\sigma$ term in HAL is a heuristic exploration bonus, whereas our $+\frac{\lambda^2}{2}\sigma^2$ term arises from the Gaussian MGF identity $\mathbb{E}[e^{\lambda f}] = e^{\lambda\mu + \frac{\lambda^2}{2}\sigma^2}$, naturally weighting uncertainty by its relevance to the rare-event distribution.

\begin{table}[htbp]
\centering
\caption{Summary of heuristic active learning methods for interatomic potentials. Here, $\beta = -\lambda > 0$, so that $1/\beta$ corresponds, up to a constant pre-factor, to the temperature of the simulated system. $f$ is a vector-valued quantity measuring the uncertainty in force predictions of the MLIP.  }
\label{tab:al_methods}
\begin{tabular}{@{}lcc@{}}
\toprule
\textbf{Method} & \textbf{Sampling Distribution} & \textbf{Uncertainty measure} \\
\midrule
FLARE \citep{vandermause2020fly} & $\exp(-\beta \mu)$ & $\UC=\sigma$\\[4pt]
HAL \citep{van2023hyperactive} & $\exp(-\beta(\mu - \tau\sigma))$ & $\UC = \max {\rm softmax}(f)$ \\[4pt]
UDD-AL \citep{kulichenko2023uncertainty} & $\exp(-\beta(\mu + E_{\text{bias}}(\sigma)))$ & $\UC = \sigma$ \\[4pt]
\bottomrule
\end{tabular}
\end{table}

\subsection{Discrete Experimental Results}
\label{App: discrete_exp}

We evaluate the same benchmark functions on discretized domains, where $\mathcal{X}$ consists of a finite grid of candidate points. This setting allows exact computation of the constraint threshold $\mathbb{E}_{\tilde{p}_{t-1}}[\sigma_{t-1}^2]$ without Monte Carlo approximation, corresponding to the $\epsilon_{\mathrm{MC}} = 0$ case in our theoretical analysis.

Results are shown in Figure~\ref{fig:synthetic_exp_discreat}. The overall trends mirror the continuous setting: \texttt{AB-SID-iVAR}-$\bar{\mathcal{X}}$ consistently achieves the lowest weighted MSE across benchmarks, with the gap widening in higher dimensions. The relative ordering of methods remains largely unchanged from the continuous experiments (Figure~\ref{fig:synthetic_exp}), confirming that our findings are not artifacts of MCMC approximation quality.

Notably, the MC variant \texttt{TS-SID-iVAR} remains competitive in general but exhibits high variance on multimodal targets such as Branin. This occurs because occasionally, when the Thompson sample concentrates on a subset of modes, the constraint threshold $\mathbb{E}_{\tilde{p}_{t-1}}[\sigma_{t-1}^2]$ becomes small, resulting in a vast feasible region $\bar{\mathcal{X}}_t$ that fails to encourage sufficient exploration. This empirical finding does not contradict our theoretical guarantees (Theorem~\ref{thm:hp_terminal}), which are asymptotic in high probability as $T \to \infty$. In discrete domains, this can be empirically mitigated by adopting a more conservative threshold such as the domain-averaged variance.

\begin{figure*}[h!]
    \centering
\includegraphics[width=1.0\linewidth]{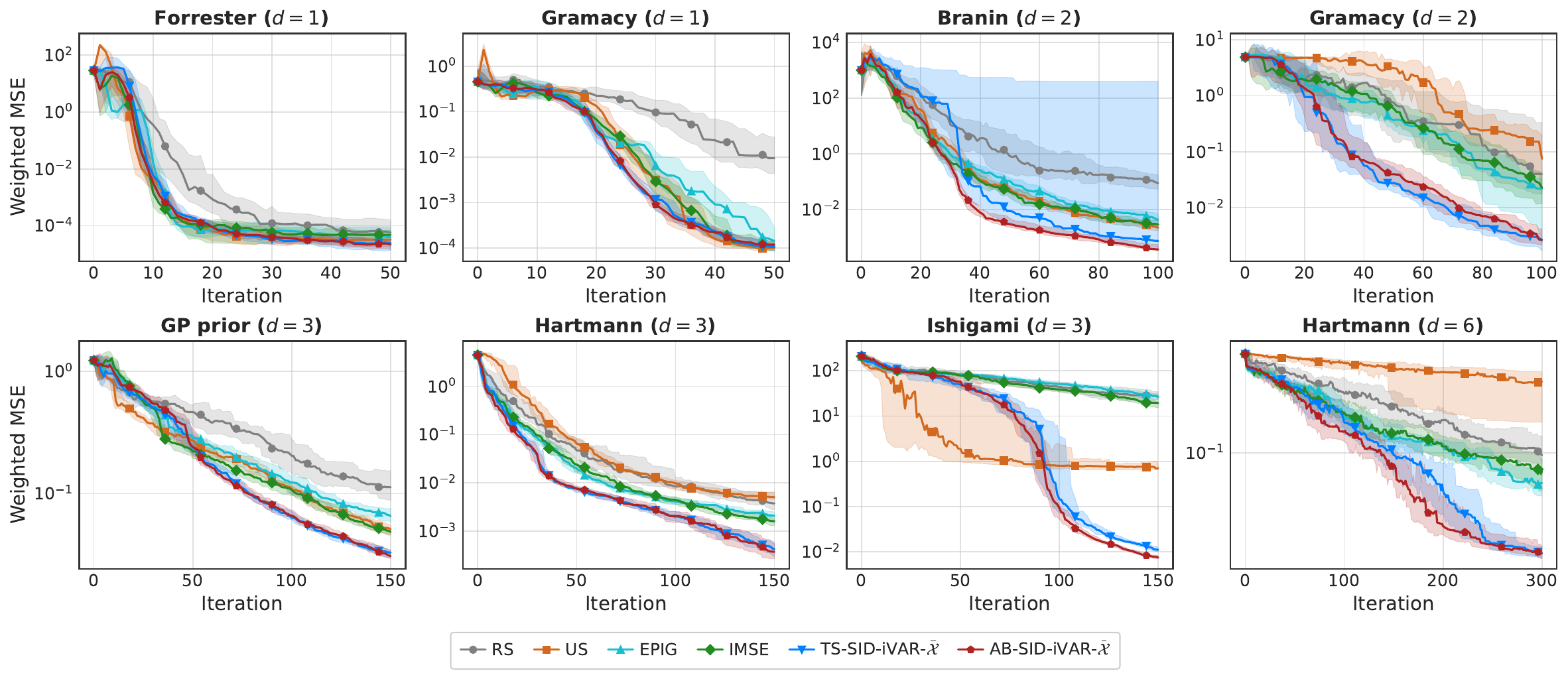}
    \caption{Weighted MSE convergence on synthetic benchmarks with discrete input domains across dimensions $d = 1$ to $d = 6$. Each panel shows median performance over 30 random seeds (shaded regions: interquartile range). Results mirror the continuous setting (Figure~\ref{fig:synthetic_exp}), with \texttt{AB-SID-iVAR}-$\bar{\mathcal{X}}$ consistently achieving the lowest weighted MSE.}
    \label{fig:synthetic_exp_discreat}
\end{figure*}

\subsection{Ablation Study}
\label{App: Ablation_Study}
\begin{figure}[h!]
    \centering
    \includegraphics[width=1.0\linewidth]{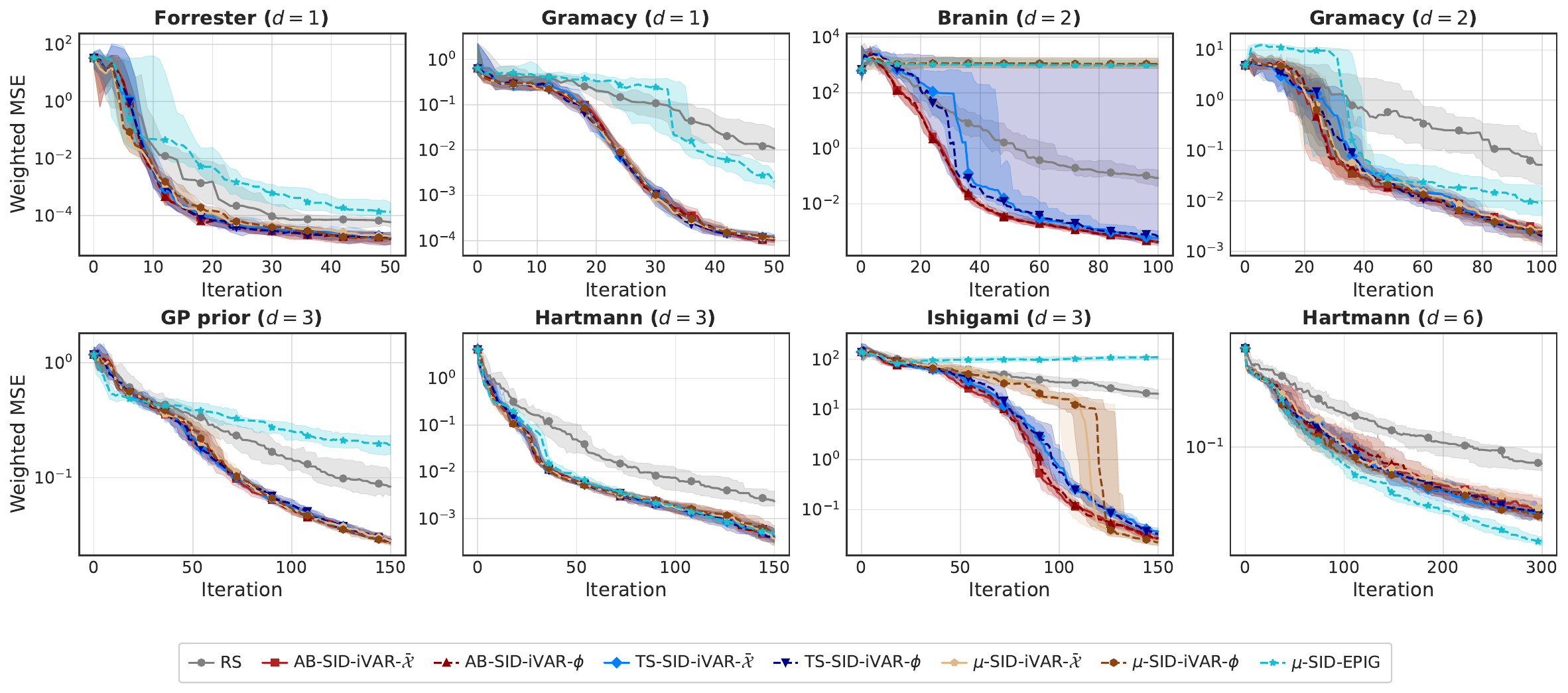}
    \caption{Ablation study comparing surrogate forms, acquisition strategies, and constraint sets.
$\bar{\mathcal{X}}$: methods with the variance-threshold constraint; $\phi$: unconstrained
optimization over $\mathcal{X}$. The $\mu$-SID family uses the plug-in surrogate
$e^{\lambda\mu_{t-1}}$; $\mu$-SID-EPIG replaces the iVAR criterion with EPIG on the same
surrogate. Random sampling (RS) is included as a baseline. On problems dominated by a
broad basin of comparable local minima (Hartmann $d=3,6$), plug-in $\mu$-SID variants achieve
competitive convergence. In more general settings, particularly when the target distribution
has well-separated modes (Branin) or strongly oscillatory structure (Ishigami), AB/TS-SID-iVAR
yield more robust convergence. }
    \label{fig:abl_study}
\end{figure}

We compare: (1) the surrogate form (\texttt{AB-SID} and \texttt{TS-SID} versus the
naive plug-in $\mu$-SID using $e^{\lambda\mu_{t-1}}$), tested with both our \texttt{iVAR}
acquisition (denoted $\mu$-\texttt{SID-iVAR}) and EPIG (denoted $\mu$-\texttt{SID-EPIG})
to verify that the choice of acquisition function alone cannot compensate for a plug-in
surrogate, and (2) the presence of the constraint set $\bar{\mathcal{X}}$ versus
unconstrained optimization over full $\mathcal{X}$ (denoted $\phi$). Results are shown
in Figure~\ref{fig:abl_study}.

On functions where the target distribution does not exhibit strongly separated modes
(Forrester, Gramacy 1D, Hartmann), all variants perform comparably. However, when the
target distribution concentrates on well-separated regions or exhibits oscillatory
structure, clear failure modes emerge. On Branin, where the target distribution $P_f$
is bimodal, both $\mu$-\texttt{SID-iVAR} and $\mu$-\texttt{SID-EPIG} fail entirely as
the plug-in surrogate $e^{\lambda\mu_{t-1}}$ commits to whichever mode is discovered
first and cannot explore alternatives, resulting in weighted MSE that remains near its
initial value throughout. On Ishigami, both plug-in variants converge substantially
later than \texttt{AB/TS-SID-iVAR}. \texttt{AB-SID} and \texttt{TS-SID} avoid these
failures by incorporating posterior uncertainty into the surrogate via the
variance-inflation term $\frac{\lambda^2}{2}\sigma^2_{t-1}$, allowing probability mass
to spread across uncertain regions.

The constraint set $\bar{\mathcal{X}}$ provides complementary robustness. On Branin
and Ishigami, the unconstrained variants ($\phi$) exhibit higher variance across seeds,
as they may exploit the current surrogate's modes without sufficient exploration.

\subsection{Real World Problems}

\subsubsection{Potential Energy Surface Modeling} \label{App: PES_modeling_details}

In all four systems below we report $\lambda$; for $\lambda<0$, the physical
temperature relates to $\lambda$ via $\mathcal{T} = -(k_{\rm B}\lambda)^{-1}$,
where $k_{\rm B} = 8.617 \times 10^{-5}$\,eV/K. Energies are reported in eV
throughout.

\paragraph{$\mathbf{H_2}$ on $\mathbf{Cu(100)}$.}
We fit the potential energy surface of a hydrogen molecule, $\mathrm{H}_{2}$,
on top of a $\mathrm{Cu(100)}$ surface. The chemistry of primary interest is
the dissociative adsorption of $\mathrm{H}_{2}$ on the metal surface \cite{smits2023quantum}, which
is well described by two reaction coordinates: the molecule's distance from
the surface, $z$, and the interatomic distance between the two H atoms, $d$.
We fit $E$ as a function of $(z, d)$. As ground truth we use a custom
MACE model~\cite{stark2024benchmarking} trained on $\mathrm{H}_{2}$ reactive
surface configurations. The simulation cell follows the setup of~\cite{stark2023machine,stark2025nonadiabatic}: a single $\mathrm{H}_{2}$ molecule above a $3\times 3$, six-layer
$\mathrm{Cu(100)}$ slab consisting of 54 Cu atoms in face-centered cubic (FCC)
lattice structure, within a periodic
$7.749 \times 7.748 \times 49.132\,\text{\AA}^3$ unit cell. We use a
conservative weighting $|\lambda| = 1$ (treated as an abstract concentration
parameter rather than a physical temperature), yielding broad coverage of
both the bonded and dissociated states and the barrier between them.

\paragraph{$\mathbf{H}$ on $\mathbf{Cu_{13}}$.}
We consider the adsorption of a hydrogen probe atom on a frozen 13-atom
copper cluster in the icosahedral global-minimum geometry~\cite{yang2006structure}
in vacuum. We parametrise the probe position in spherical coordinates
$(r, u, \phi)$ centered at the cluster centroid,
\[
\mathbf{x} = \bigl(r\sqrt{1-u^{2}}\cos\phi,\ r\sqrt{1-u^{2}}\sin\phi,\ ru\bigr)
\in \mathbb{R}^3,
\]
where $u = \cos\theta$ is used in place of the polar angle $\theta$ so that
uniform sampling on $(u, \phi)$ induces the uniform measure on the sphere.
We restrict $r \in [0.5, 6.0]\,\text{\AA}$ to cover surface and near-surface
configurations, and exploit the icosahedral symmetry of the cluster by
restricting $(u, \phi)$ to a single fundamental domain
$[0.7947, 1] \times [0, \pi/5]$. Energies are evaluated using the
MACE-MPA-0 foundation model~\cite{batatia2025foundation}. We use
$|\lambda| = 8.55\,\text{eV}^{-1}$, corresponding to $\mathcal{T} \approx 1358$\,K
(close to the melting point of copper), at which the Boltzmann target
concentrates on the cluster's surface adsorption sites.

\paragraph{$\mathbf{H_2O}$ on $\mathbf{Pt}$.}
The energy landscape of an $\mathrm{H_2O}$ molecule on a platinum surface features
multiple physisorption wells, chemisorption minima, and dissociation
pathways \cite{schnur2009properties}. We model the metal surface as a three-layer ABC-stacked $\mathrm{Pt(111)}$ slab
in a $3\times 3$ in-plane supercell with two-dimensional periodic boundary
conditions ($\sim 8\,\text{\AA}$ lateral spacing between water images).
A single rigid $\mathrm{H_2O}$ molecule is placed above the slab and
parametrised by 6 degrees of freedom: the oxygen's lateral position $(x, y)$
within the $p3m1$ fundamental triangle of the surface unit cell, its height
$z \in [0.5, 5.0]\,\text{\AA}$ above the topmost Pt layer, and the molecular
orientation given by ZYZ Euler angles $(\alpha, \cos\beta, \gamma)$. As in
the spherical parametrisation above, $\cos\beta$ replaces $\beta$ for uniform
sampling on $SO(3)$; $\gamma$ is restricted to $[0, \pi)$ to account for the
$C_{2v}$ symmetry of $\mathrm{H_2O}$. Energies are evaluated using the
MACE-MPA-0 foundation model~\cite{batatia2025foundation}. We use
$|\lambda| = 38.68\,\text{eV}^{-1}$, corresponding to room temperature
$\mathcal{T} = 300$\,K, the standard condition for water-on-platinum
experiments.

\paragraph{Si Crystal.}
We fit the potential energy of silicon in a highly symmetric diamond cubic structure,
modeled by two atoms in a cubic unit cell
with periodic boundary conditions. By translation invariance, the relevant
degrees of freedom reduce to the displacement vector between the two atoms. As ground truth we use the
Stillinger--Weber potential~\cite{stillinger1985computer}, a canonical
empirical potential for tetrahedrally-bonded silicon. We use
$|\lambda| = 11.60\,\text{eV}^{-1}$, corresponding to
$\mathcal{T} = 1000$\,K, a typical Si annealing / chemical vapor deposition
temperature, so that the Boltzmann distribution covers a wide range of
thermalised crystal states. We emphasise that this setup is a minimal benchmark: the small cell, fixed lattice vectors, and absence of defects restrict the configuration space to thermal fluctuations around a single crystalline minimum.

\subsubsection{Molecular Drug Discovery}\label{App: Mol_drug_disc}

The candidate pool consists of approximately 20k molecules from the GuacaMol benchmark~\citep{brown2019guacamol}, each represented by a 2048-bit Morgan fingerprint. We use a GP with a Tanimoto kernel as our surrogate model and consider two commonly used multi-property drug-likeness scores (Median 1 and Median 2) as learning objectives, testing $\lambda \in \{25, 75\}$ to vary the concentration of the target distribution (Figure~\ref{fig:molecular_drug_discovery}, top row). This setting lies outside our theoretical assumptions as the Tanimoto kernel is neither stationary nor differentiable; nevertheless, the problem is of practical interest and serves as a challenging test of whether \texttt{AB-SID-iVAR} remains effective beyond the provable regime. We omit TS-SID-iVAR as drawing posterior samples over the full 20k-point pool is computationally prohibitive, though efficient approximations exist for Tanimoto kernels~\citep{tripp2023tanimoto}. We also omit \texttt{GHAL} as it is designed for continuous domains.

\subsection{Runtime Report and Compute Resources}\label{app: run_time_rep}

\paragraph{Compute setup.} All experiments were run on an internal SLURM cluster. Each run was allocated 4 CPU cores and 32~GB of RAM on internal computational nodes without usage of any GPUs.

\paragraph{Per-iteration runtime.}
We report the per-iteration runtime (in seconds) for each method in Table~\ref{tab:runtime_profile}, averaged across seeds. For our methods (AB-$\bar{\mathcal{X}}$ and TS-$\bar{\mathcal{X}}$), we decompose the runtime into three components: GP fitting, acquisition
  optimization, and SMC sampling.

Several observations can be made from this analysis. First, GP fitting contributes negligibly to total runtime across all methods. Second, acquisition optimization dominates the computational cost for our methods. This is expected: optimizing the SID-iVAR objective requires evaluating the fantasy variance at all reference points for each candidate, and the $\bar{\mathcal{X}}$ constraint introduces an additional posterior evaluation per optimization step, effectively doubling the per-step cost compared to methods like IMSE, which use an unconstrained L-BFGS-B optimizer. Third, SMC sampling adds modest overhead in continuous mode but is not required for discrete domains where the candidate set is given.

\begin{table}[htbp]
\centering
\caption{Runtime profile (seconds per iteration, mean $\pm$ std across seeds).}
\label{tab:runtime_profile}
\small
\begin{adjustbox}{max width=\textwidth}
\begin{tabular}{@{}lcccccccccc@{}}
\toprule
 & $d$ & IMSE & AB-$\bar{\mathcal{X}}$ & \multicolumn{3}{c}{AB-$\bar{\mathcal{X}}$ Breakdown} & TS-$\bar{\mathcal{X}}$ & \multicolumn{3}{c}{TS-$\bar{\mathcal{X}}$ Breakdown} \\
\cmidrule(lr){5-7} \cmidrule(lr){9-11}
 &  &  &  & GP & Acq Opt & SMC &  & GP & Acq Opt & SMC \\
\midrule
Continuous & 1 & 0.30$\pm$0.04 & 2.3$\pm$0.6 & 0.06$\pm$0.01 & 2.2$\pm$0.6 & 0.62$\pm$0.16 & 3.6$\pm$1.2 & 0.05$\pm$0.01 & 3.5$\pm$1.1 & 1.9$\pm$0.7 \\
 & 2 & 0.68$\pm$0.09 & 3.7$\pm$1.1 & 0.06$\pm$0.02 & 3.3$\pm$1.0 & 0.84$\pm$0.32 & 6.0$\pm$1.8 & 0.06$\pm$0.01 & 5.6$\pm$1.7 & 3.2$\pm$1.1 \\
 & 3 & 1.5$\pm$0.2 & 6.1$\pm$1.2 & 0.08$\pm$0.01 & 5.2$\pm$1.2 & 1.2$\pm$0.3 & 12.8$\pm$2.3 & 0.09$\pm$0.02 & 11.9$\pm$2.2 & 7.2$\pm$1.4 \\
 & 6 & 4.9$\pm$0.7 & 16.3$\pm$3.4 & 0.15$\pm$0.03 & 13.3$\pm$2.9 & 3.6$\pm$0.8 & 26.2$\pm$5.0 & 0.14$\pm$0.03 & 23.3$\pm$4.5 & 14.0$\pm$2.7 \\
\midrule
Discrete & 1 & 2.4$\pm$0.6 & 3.7$\pm$0.5 & 0.06$\pm$0.01 & 3.6$\pm$0.5 & NA & 4.0$\pm$0.5 & 0.05$\pm$0.01 & 3.8$\pm$0.5 & NA \\
 & 2 & 3.2$\pm$0.5 & 5.3$\pm$1.1 & 0.05$\pm$0.01 & 5.2$\pm$1.1 & NA & 4.6$\pm$2.4 & 0.06$\pm$0.02 & 4.4$\pm$2.4 & NA \\
 & 3 & 2.6$\pm$0.7 & 5.9$\pm$1.3 & 0.08$\pm$0.02 & 5.7$\pm$1.3 & NA & 6.1$\pm$1.1 & 0.08$\pm$0.01 & 5.9$\pm$1.1 & NA \\
 & 6 & 0.81$\pm$0.16 & 1.8$\pm$0.3 & 0.16$\pm$0.03 & 1.5$\pm$0.3 & NA & 2.0$\pm$0.4 & 0.15$\pm$0.03 & 1.7$\pm$0.3 & NA \\
\bottomrule
\end{tabular}
\end{adjustbox}
\end{table}

\subsection{Sensitivity Analysis of MCMC sample size} \label{app: sensitivity_analysis_of_mcmc}

We investigate the sensitivity of our methods to the SMC particle size $N$ on GP prior
functions across $d \in \{2, 4, 6, 8\}$, testing $N \in \{250d, 500d, 1000d, 1500d\}$.
Figure~\ref{fig:mcmc_runtime_report} reports the weighted MSE convergence (top) and
per-iteration MCMC runtime (bottom). Convergence is essentially unchanged across particle
counts for both AB-$\bar{\mathcal{X}}$ and TS-$\bar{\mathcal{X}}$, while runtime scales
roughly linearly with $N$. TS-$\bar{\mathcal{X}}$ incurs higher cost than AB-$\bar{\mathcal{X}}$
at matched $N$ due to pathwise-conditioned posterior sampling. This justifies $N = 250d$
as a sufficient default in practice.

\begin{figure}[h!]
    \centering
    \includegraphics[width=1.0\linewidth]{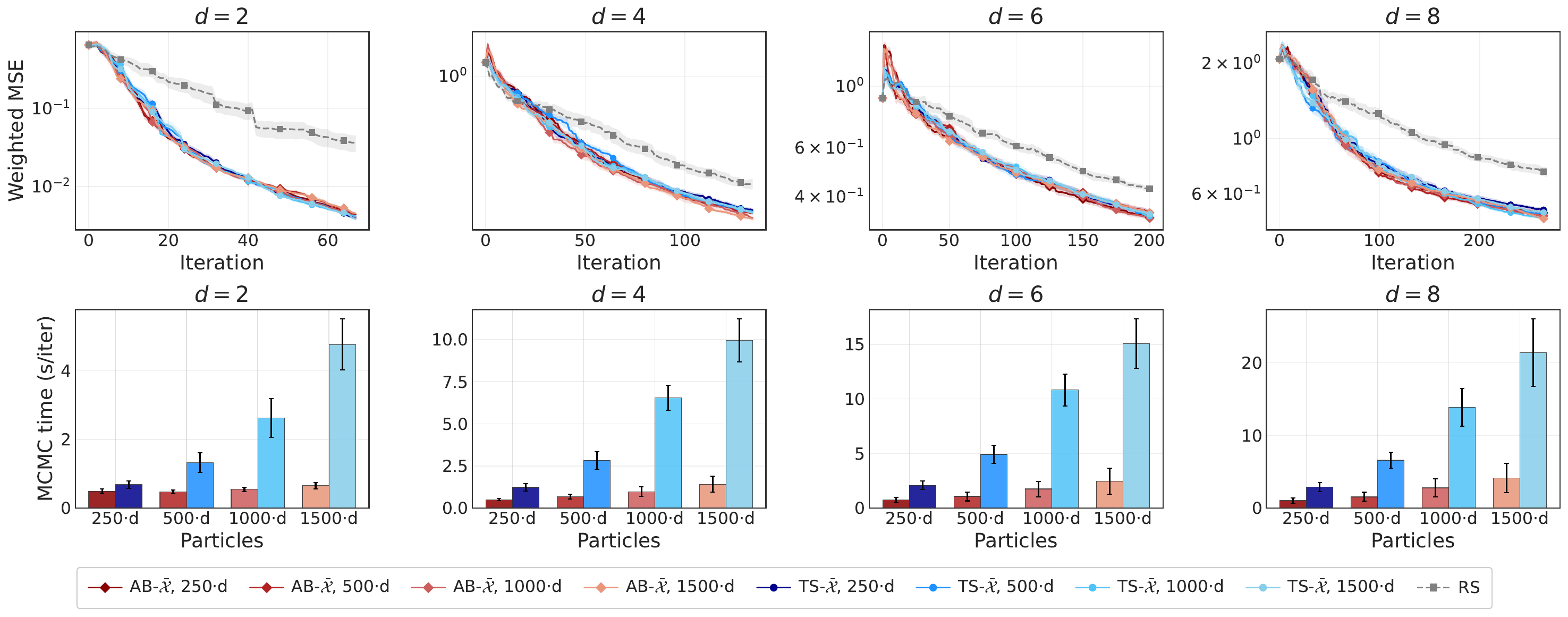}
    \caption{Sensitivity of SMC particle size $N \in \{250d, 500d, 1000d, 1500d\}$ on
    GP prior functions across $d \in \{2,4,6,8\}$. Top: weighted MSE convergence; curves
    overlap across particle counts. Bottom: per-iteration MCMC time, scaling linearly
    with $N$.}
    \label{fig:mcmc_runtime_report}
\end{figure}

\subsection{Sensitivity Analysis of $\lambda$ and $b$}\label{app: temp_and_bias}
\begin{figure}[h!]
    \centering
    \includegraphics[width=1.0\linewidth]{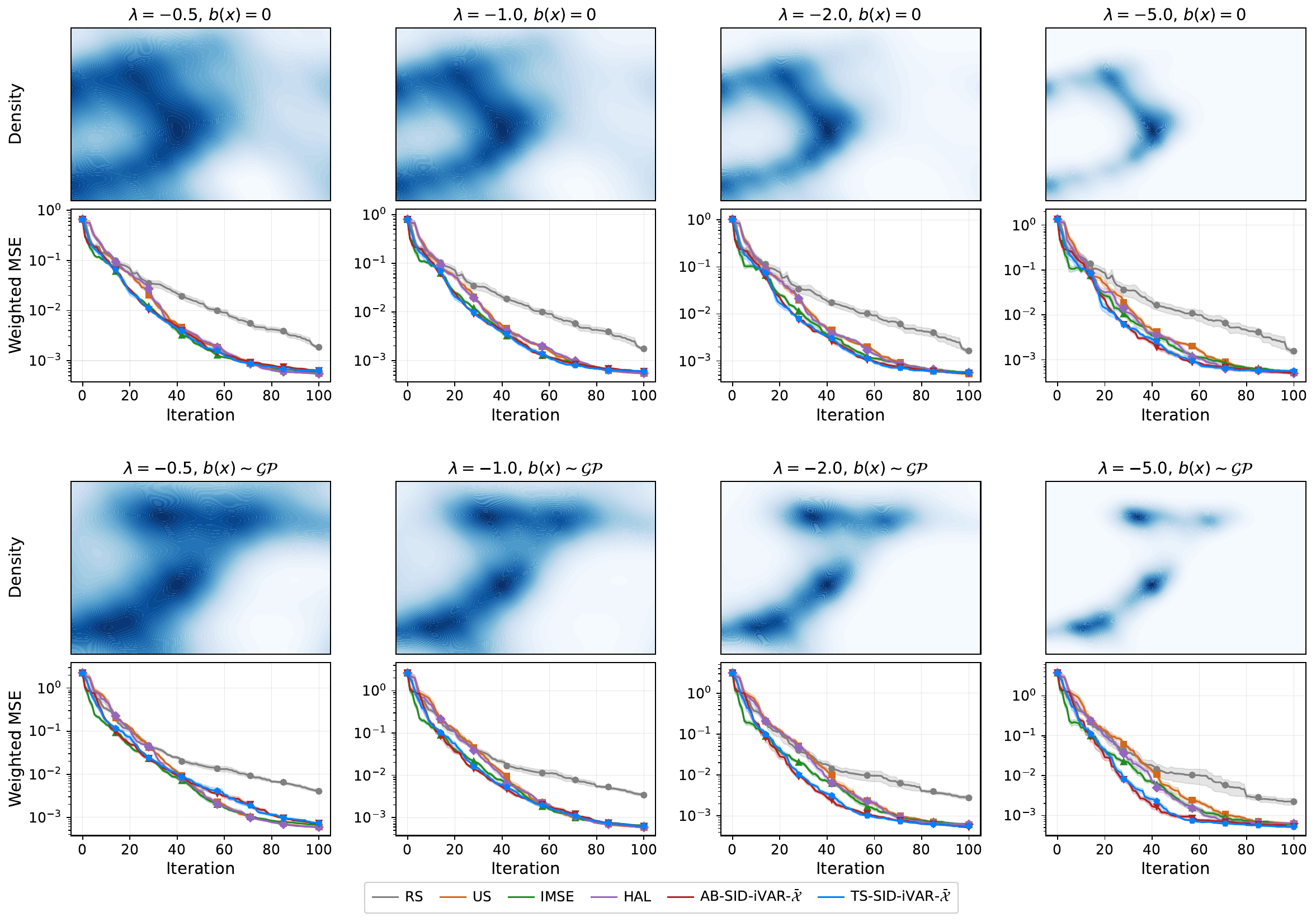}
    \caption{Sensitivity analysis of temperature $\lambda$ and bias $b(x)$ on 2-dimensional GP prior samples. Top (two) row: constant bias $b(x) = 0$; bottom (two) row: function form bias $b(x) \sim \mathcal{GP}$. Columns correspond to $\lambda \in \{-0.5, -1.0, -2.0, -5.0\}$ from left to right. Density plots show the ground truth with high-probability regions under $P_f \propto \exp(\lambda f + b)$ shaded in \textcolor{cyan}{blue}. As $|\lambda|$ increases, the target concentrates and SID-aware methods show greater relative advantage.}
    \label{fig:sensitivity_analysis_of_lambda_and_b}
\end{figure}

We investigate sensitivity to the temperature parameter $\lambda$ and bias function $b(x)$ on functions drawn from a 2 dimensional GP prior across $\lambda \in \{-0.5, -1.0, -2.0, -5.0\}$ with either constant or function form bias $b(x)$ drawn from an independent GP prior.

The results are illustrated in Figure~\ref{fig:sensitivity_analysis_of_lambda_and_b}. As $\vert\lambda\vert$ increases, the target distribution $P_f \propto \exp(\lambda f + b)$ concentrates around the minima of $f$. Under weak concentration ($\lambda = -0.5$), all methods exhibit comparable performance. However, the performance benefit gets more obvious as $|\lambda|$ grows. This reflects the increasing inefficiency of uniform or uncertainty-driven sampling when the evaluation metric concentrates on a shrinking subset of the domain.

The introduction of spatially-varying bias $b(x) \sim \mathcal{GP}$ (bottom row) produces more complex target distributions, yet the relative ordering among methods is preserved, confirming that our framework extends beyond the $b(x) = 0$ setting assumed in the main experiments.

\subsection{Molecular Drug Discovery}

To complement the quantitative $R^2$ analysis, Figure~\ref{fig:pairwise_plot} visualizes prediction quality across the entire molecule pool. Each panel plots the GP posterior mean $\mu$ against the true score $y$, with points colored by their Boltzmann weight under $P_\lambda$. This visualization directly reveals \emph{where} in function-value space each method succeeds or fails.

For \texttt{AB-SID-iVAR}, dark points (high Boltzmann weight) cluster tightly around the diagonal across all settings, indicating that the model has learned the high-value region accurately. In contrast, baseline methods exhibit two failure modes visible in the plots: (i) vertical spread of dark points, indicating high variance in predictions for task-relevant molecules, and (ii) systematic bias where dark points deviate from the diagonal, indicating the model has not adequately explored these regions. These deficiencies become more severe at $\lambda = 75$, where the target distribution concentrates on a narrower slice of high-scoring compounds. The pairwise plots thus provide a complementary diagnostic to the $R^2$-on-top-$k\%$ metric, confirming that \texttt{AB-SID-iVAR}'s advantage stems from concentrating sampling effort on the Boltzmann-weighted region of interest.
\begin{figure}
    \centering
    \includegraphics[width=1.0\linewidth]{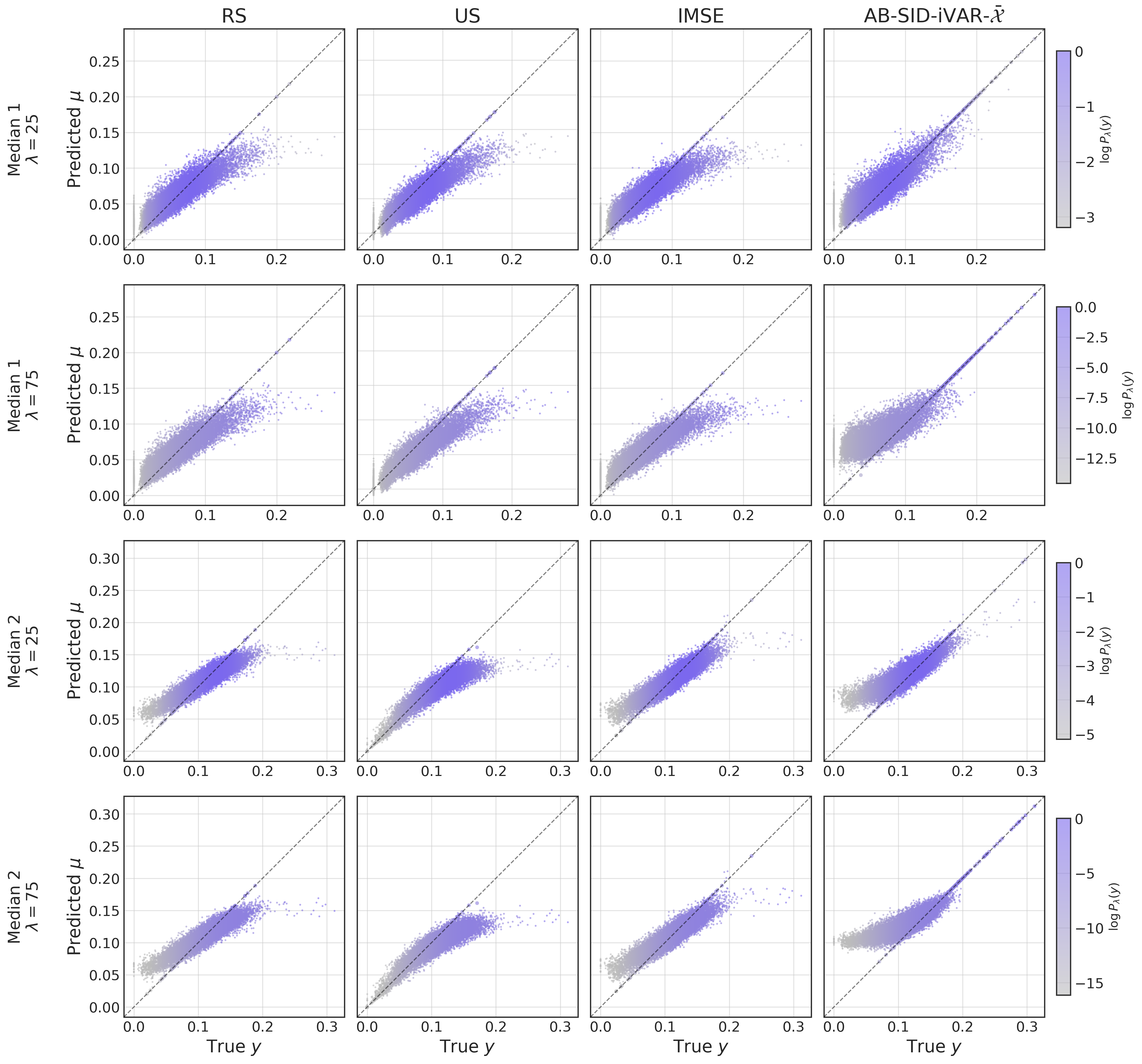}
    \caption{Pairwise comparison of predicted GP posterior mean $\mu$ versus true function value $y$ across all $\approx$20k GuacaMol molecules. Rows: scoring functions (Median 1, Median 2) and Boltzmann temperatures ($\lambda = 25, 75$). Columns: acquisition strategies. Point color encodes log Boltzmann weight $\log P_\lambda(y) \propto \lambda y$ (darker $=$ higher weight, i.e., more task-relevant molecules), we omit reporting the result of EPIG here for the sake of space, and it behaves almost identical to IMSE. Ideal performance corresponds to dark points clustering tightly along the $y = \mu$ diagonal. \texttt{AB-SID-iVAR} (rightmost column) achieves this across all settings, whereas baselines show systematic vertical spread or bias for high-weight molecules---particularly pronounced at $\lambda = 75$ where the target distribution is sharply concentrated.}
    \label{fig:pairwise_plot}
\end{figure}

\section{Limitations and Future Work} \label{App:limit}

Our theoretical bounds exhibit exponential dependence on $|\lambda|$ through the density-ratio argument (Lemmas~\ref{lem:log_density_ratio}--\ref{lem:hp_mse_bound}). While this establishes consistency for any fixed $\lambda$, the global nature of this approach, requiring uniform control of the partition function ratio across the entire domain, introduces an exponential prefactor $e^{\mathcal{O}(|\lambda|\sqrt{\beta_T})}$ that does not appear in analogous Bayesian optimization results for the $\lambda \to \infty$ limit.

Computationally, our acquisition function relies on MCMC sampling to approximate the constraint set $\bar{\mathcal{X}}_t$ and the integral in Eq.~\ref{Eq: SID-iVAR} over continuous domains. While Figure~\ref{fig:mcmc_runtime_report} shows that substantially smaller SMC sample sizes than those used in our experiments suffice without degrading performance, thereby reducing acquisition optimization cost, this still introduces overhead beyond methods that require only point evaluations of the surrogate. Since our algorithm is agnostic to the choice of MCMC sampler, we believe this cost can be further reduced via amortized Boltzmann distribution samplers (e.g., ~\citep{ouyang2024bnem, akhound2025progressive}).

Several directions are promising for future work. Empirically, our evaluation is limited to controlled scientific benchmarks,
with continuous-domain experiments covering low to moderate input
dimensionalities. Validation in more complex, end-to-end scientific
applications remains an important future work. On the theoretical side, developing input-dependent analyses that directly track coverage of high-probability regions, bridging our analysis with Bayesian optimization regret bounds in the $\lambda \to \infty$ limit, and extending guarantees to non-Boltzmann SID forms are natural next steps. Additionally, our preliminary analysis of \texttt{GHAL} (Proposition~\ref{prop:fixed-threshold}) assumes exhaustive exploration of the input space by the Markov chain; relaxing this condition to cover more general settings remains a natural direction of future study. Finally, our acquisition function couples the AB/TS-SID surrogate with the iVAR criterion as an integrated design, motivated by and analyzed through the variance-based upper bounds in Lemmas~\ref{lem:hp_mse_bound}-\ref{lem:bayesian_mse}. Studying alternative integrated designs under the SIDAL framework, along with their corresponding theoretical guarantees (which would require analysis distinct from Theorem~\ref{thm:hp_terminal}), is a promising direction for future work.

\end{document}